\documentclass[10pt]{article} % For LaTeX2e
\usepackage[preprint]{rlc}
\usepackage{amssymb}            % Defines common symbols like \mathbb R
\usepackage{mathtools}          % Extends amsmath, providing common math tools
\usepackage{mathrsfs}           % Enables \mathscr, which can work in cases that \mathcal does not
\usepackage[american]{babel}
\usepackage{algorithm}
\let\classAND\AND
\let\AND\relax
\usepackage{algorithmic}
\let\AND\classAND
\AtBeginEnvironment{algorithmic}{\let\AND\algoAND}

\floatname{algorithm}{Algorithm}

\usepackage{subcaption}
\usepackage[disable]{todonotes}

\usepackage{mathtools} % amsmath with fixes and additions
\usepackage{booktabs} % commands to create good-looking tables
\usepackage{tikz} % nice language for creating drawings and diagrams
\usepackage{amsmath,amsfonts,bm, amsthm}

\newtheorem{theorem}{Theorem}
\newtheorem{lemma}{Lemma}
\newtheorem{proposition}{Proposition}
\newtheorem{assumption}{Assumption}

\def\S{{\mathcal{S}}}
\def\A{\mathcal{A}}
\def\P{\mathcal{P}}
\def \EE{\mathbb{E}}

\def\D{{D_{\pi_\theta}}}

\def\eqref#1{equation~\ref{#1}}
\def\1{\bm{1}}

\DeclareMathAlphabet{\mathsfit}{\encodingdefault}{\sfdefault}{m}{sl}
\SetMathAlphabet{\mathsfit}{bold}{\encodingdefault}{\sfdefault}{bx}{n}

\newcommand{\E}{\mathbb{E}}

\newcommand{\R}{\mathbb{R}}

\DeclareMathOperator*{\argmax}{arg\,max}
\DeclareMathOperator*{\argmin}{arg\,min}

\usepackage{amsmath}
\usepackage{amssymb}
\usepackage{mathtools}
\usepackage{amsthm}
\theoremstyle{plain}

\usepackage[utf8]{inputenc} % allow utf-8 input
\usepackage[T1]{fontenc}    % use 8-bit T1 fonts
\usepackage{hyperref}       % hyperlinks
\usepackage{url}            % simple URL typesetting
\usepackage{booktabs}       % professional-quality tables
\usepackage{amsfonts}       % blackboard math symbols
\usepackage{nicefrac}       % compact symbols for 1/2, etc.
\usepackage{microtype}      % microtypography
\usepackage{xcolor}         % colors

\usepackage{soul}

\newcommand{\Prob}{\mathbb{P}}

\title{Towards Provable Log Density Policy Gradient}

\author{%
  Pulkit Katdare\thanks{Equal contribution.}\\
  Department of Electrical and Computer Engineering \\ 
  University of Illinois at Urbana-Champaign\\
  \texttt{katdare2@illinois.edu} 
  \AND
  Anant A. Joshi$^{*}$ \\
  Department of Mechanical Engineering\\
  University of Illinois at Urbana-Champaign
  \AND
  Katherine Driggs-Campbell \\
  Department of Electrical and Computer Engineering\\
  University of Illinois at Urbana-Champaign \\
}

\begin{document}

\maketitle

\begin{abstract}
Policy gradient methods are a vital ingredient behind the success of modern reinforcement learning. Modern policy gradient methods, although successful, introduce a residual error in gradient estimation. 
In this work, we argue that this residual term is significant and correcting for it could potentially improve sample-complexity of reinforcement learning methods. To that end, we propose log density gradient to estimate the policy gradient, which corrects for this residual error term. 
Log density gradient method computes policy gradient by utilising the state-action discounted distributional formulation. 
We first present the equations needed to exactly find the log density gradient for a tabular Markov Decision Processes (MDPs). 
% Further, we show that such a method for estimating policy gradient always gives us a unique solution irrespective of the discounting factor. 
For more complex environments, we propose a temporal difference (TD) method that approximates log density gradient by utilizing backward on-policy samples. 
Since backward sampling from a Markov chain is highly restrictive we also propose a min-max optimization that can approximate log density gradient using just on-policy samples. 
We also prove uniqueness, and convergence under linear function approximation, for this min-max optimization. Finally, we show that the sample complexity of our min-max optimization to be of the order of $m^{-1/2}$, where $m$ is the number of on-policy samples. 
We also demonstrate a proof-of-concept for our log density gradient method on gridworld environment, and observe that our method is able to improve upon the classical policy gradient method by a clear margin, thus indicating a promising novel direction to develop reinforcement learning algorithms that require fewer samples. 
\end{abstract}
\section{Introduction}
%Reinforcement learning (RL) algorithms have been successful in achieving state-of-the-art performance over a wide-variety of complex tasks ~\cite{DBLP:journals/nature/SilverSSAHGHBLB17, chatgpt}. RL problems are typically classified as off-policy \cite{DBLP:journals/corr/LillicrapHPHETS15, DBLP:journals/corr/MnihKSGAWR13} or on-policy \cite{DBLP:conf/nips/SuttonMSM99, DBLP:journals/corr/SchulmanWDRK17, DBLP:journals/corr/SchulmanLMJA15}. Off-policy methods update policy while collecting data and require minimizing bellman error objective. On the other hand, On-policy methods use samples from the current policy and update policy by approximating the gradient of the RL objective with respect to the policy parameters~\cite{DBLP:conf/nips/SuttonMSM99}. Such a gradient is often referred to as, policy gradient. 
%\todo[inline]{change cite to citep where ever appropriate}
Policy gradient (PG) methods are a vital ingredient behind the success of modern reinforcement learning~\citep{DBLP:journals/nature/SilverSSAHGHBLB17, chatgpt,  DBLP:conf/icml/HaarnojaZAL18, DBLP:conf/nips/Kakade01}. 
The success of PG methods stems from their simplicity and compatibility with neural network-based function approximations~\citep{DBLP:conf/nips/SuttonMSM99, pg_baxter}. 
Although modern policy gradient methods like PPO and TRPO, have achieved excellent results in various on-policy tasks~\citep{DBLP:journals/corr/SchulmanWDRK17, DBLP:journals/corr/SchulmanLMJA15}, they require extensive hyper-parameter tuning.
Additionally, it has been shown by ~\citet{DBLP:conf/iclr/IlyasESTJRM20} that the estimation error between policy gradient estimated by the methods like PPO and the true policy gradient increases significantly during the training process. 
% In this paper, we will show that 
Classical policy gradient methods typically approximate gradient of the policy using Q-function estimated with discount factor strictly less than 1, which leads to a error in gradient estimation~\citep{DBLP:journals/neco/MorimuraUYPD10}. 
% in this figure we demonstrate theat this res error is indeed dignificant
% In this paper, we will show that classical policy gradient methods typically require many approximations in order to make gradient estimations computationally feasible. We observe that this kind of approximation leads to a error in gradient estimation~\citep{DBLP:journals/neco/MorimuraUYPD10}. 
In this paper, we demonstrate that this error in indeed significant, in Figure \ref{fig:motivation}. 
We propose a novel algorithm to correct for this error which could potentially lead to a sample efficient reinforcement learning, thus enabling their deployment over a wide variety of complex scenarios. 
We call our method \textbf{log density gradient}. 
Our approach is based on the average state-action stationary distribution formulation of reinforcement learning, which allows for the estimation of policy gradient as a multiplication of the log density gradient and the reward function~\citep{DBLP:conf/nips/NachumCD019, DBLP:conf/icml/UeharaHJ20}. 
This separation results in an improved correlation with the true policy gradient and requires fewer hyperparameters. We show that our method is consistent with the classical policy gradient theorem \citep{DBLP:journals/ml/Sutton88} and also prove convergence properties and sample complexity.

% Our main contributions are as follows.
% \textbf{1.} A novel method to exactly calculate policy gradient by using the gradient of the log of the state-action discounted distribution. We will also show that policy gradient estimated for $\gamma=1$ will also end up correcting for the residual error usually seen in implementation of classical policy gradient methods. 
% \textbf{2.} A model-free Temporal Difference (TD) method for approximating policy gradient. We provide proof of contraction as well as convergence. However, there is a major drawback that it requires samples from the backward Markov chain, (described in detail in the paper) which motivates the next contribution. 
% \textbf{3.} A min-max optimization which yields the gradient of log density, with proof of convergence. We additionally show their sample complexity to be of the order of $O(n^{-1/2})$ for the projected version, where $n$ is the number of on-policy samples. Our method is competitive with the sample complexity of classical vanilla policy gradient methods~\citep{DBLP:conf/aistats/YuanGL22}.
Our main contributions are as follows. \textbf{1.} A novel method to exactly calculate policy gradient by using the average state-action discounted formulation of a reinforcement learning problem. We will also show that policy gradient estimated in this manner will correct for the residual error (as shown in Figure \ref{fig:motivation}) usually seen in practical application of classical policy gradient methods. \textbf{2.} A model-free Temporal Difference (TD) method for approximating policy gradient. We provide proof of contraction as well as convergence. However, there is a major drawback that it requires samples from the backward Markov chain (described in detail in the paper) which motivates the next contribution. \textbf{3.} A min-max optimization which yields the gradient of log density, and a model free TD method to implement it, with proof of convergence. We additionally show sample complexity of the order $O(m^{-1/2})$ for the projected version of the proposed TD method, where $m$ is the number of on-policy samples. Our method is competitive with the sample complexity of classical vanilla policy gradient methods~\citep{DBLP:conf/aistats/YuanGL22}.
% \begin{enumerate}
%     \item A novel method to exactly calculate policy gradient by using the average state-action discounted formulation of a reinforcement learning problem. We will also show that policy gradient estimated in this manner will correct for the residual error usually seen in practical application classical policy gradient methods. %This correction is especially relevant for the case when the discounting factor in a reinforcement learning problem is equal to 1. 
%     \item A model-free Temporal Difference (TD) method for approximating policy gradient. We provide proof of contraction as well as convergence. However, there is a major drawback that it requires samples from the backward Markov chain (described in detail in the paper) which motivates the next contribution. 
%     \item A min-max optimization which yields the gradient of log density, and a model free TD method to implement it, with proof of convergence. We additionally show their sample complexity to be of the order of $O(n^{-1/2})$ for the projected version of the proposed TD method, where $n$ is the number of on-policy samples. Our method is competitive with the sample complexity of classical vanilla policy gradient methods~\citep{DBLP:conf/aistats/YuanGL22}.
% \end{enumerate}

Section \ref{sec:background} starts with problem formulation and motivation behind this paper. Section \ref{sec:related_work}, discusses prior work in policy gradient methods, temporal difference methods, and min-max problems in off-policy evaluation. Our main contributions are discussed in detail starting from Section \ref{sec:our_method} which starts with rigorously defining log density gradient. Additionally we also propose a TD approach to estimate log density gradient under strict reversibility assumptions, and we describe the issue caused by this assumption. In section \ref{sec:min_max}, to overcome this issue to propose a min-max variant that allows us to estimate log density gradient algorithm using empirical samples. We finally demonstrate a proof-of-concept of our algorithm in Section \ref{sec:experiments} which shows that log density gradient can be potentially sample efficient as compared to classical policy gradient methods. 
%\vspace*{-20pt}
\section{Background and Motivation}\label{sec:background}
% \todo[inline]{Make sure all assumptions of each theorem, lemma or proposition are covered}
% To elaborate on the residual errors in current policy gradient methods. It is important to define a few key terms. 
\textbf{Notation:} we let $(\cdot)^T$ denote matrix transpose, and let $e$ represent the vector of ones, the size of which would be clear from context.

We define Markov Decision Process (MDP) as a 6-tuple of $(\S, \A, \P, r, \gamma, d_0)$. Here, $\S$ is a finite state space of the MDP, $\A$ is a finite action space, $\P$ is the transition probability matrix, $r$ is the reward function and $d_0$ is the initial distribution. The reinforcement learning problems is optimise for a policy $\pi: \S \rightarrow \Delta(\A)$ that maximizes $J_\gamma(\pi)$, defined as 
\begin{align*}
    J_\gamma(\pi) &:= (1-\gamma)\EE[\sum_{t=0}^\infty \gamma^t r(s_t, a_t)|
    s_0 \sim d_0, a_t \sim \pi(\cdot | s_t), s_{t+1} \sim \mathcal{P}(\cdot | s_t, a_t)], \, \text{ for } \gamma \in [0,1) \\
% \end{align*}
% Where, $\gamma < 1$ is the discounting factor which accounts for the impact of future rewards in present decision making. When, $\gamma = 1$, $J_1(\pi)$ is called the average reward formulation. 
% \begin{align*}
    J_1(\pi) &:= \lim_{\mathcal{T} \rightarrow \infty} \EE[\frac{1}{\mathcal{T}}\sum_{t=0}^\mathcal{T} \gamma^t r(s_t, a_t),
    s_0 \sim d_0, a_t \sim \pi(\cdot | s_t), s_{t+1} \sim \mathcal{P}(\cdot | s_t, a_t) ], \, \text{ for } \gamma =1,
\end{align*}
where $\gamma \in [0,1]$ is the discounting factor which accounts for the impact of future rewards in present decision making. 
% When $\gamma < 1$, it is the discounting factor which accounts for the impact of future rewards in present decision making. 
When $\gamma = 1$, $J_1(\pi)$ the scenario is called the average reward formulation. 
Most practical problems in reinforcement learning typically aim to solve for an optimal policy $\pi^* = \arg\max_{\pi} J_1(\pi)$. Modern reinforcement learning methods aim to parameterise policy with a set of parameters $\theta \in  \mathbb{R}^{n}$, where $n$ is the dimensions of the parameter space. We refer to such paremterisation as $\pi_\theta$. This kind of parameterisation enables us search for optimal set of parameters $\theta^*$ instead of a search over $\S \times \A$ which in practice could be very large. We define
\begin{align*}
    \theta^* := \argmax_{\theta \in \mathbb{R}^n} J_1(\pi_\theta).
\end{align*}
The Q-function $Q^{\pi_\theta}_\gamma$ is commonly used function used to describe the performance of an RL agent. Q-function calculates the long term (discounted) rewards accumulated by an agent following a fixed policy $\pi_\theta$ while starting from a state $s \in \S$ and taking an action $a \in \A$
\begin{subequations}
\begin{align}
    Q_\gamma^{\pi_\theta}(s, a) &:= \mathbb{E}[\sum_{t=0}^\infty \gamma^t r(s_t, a_t) | s_0 = s, a_0 = a, a_t \sim \pi_\theta(\cdot | s_t), s_{t+1} \sim \mathcal{P}(\cdot | s_t, a_t)] \\ 
     &= r(s, a) + \gamma \EE_{s' \sim \P(\cdot | s, a), a' \sim \pi_\theta(\cdot|s')}[Q_\gamma^{\pi_\theta}(s', a')]   \label{eq:bellman}
\end{align}
\end{subequations}
Where, equation \ref{eq:bellman} is called the Bellman Equation. Bellman equation is popularly used to estimate the Q-function using just empirical data collected on the MDP. Q-function approximation methods typically use $\gamma < 1$ for stable estimation of the Q-function. We also similarly define value function $V^{\pi_\theta}_\gamma(s) = \mathbb{E}_{a \sim \pi_\theta(\cdot | s)}[Q^{\pi_\theta}_\gamma(s, a)]$.

Modern RL algorithms generally solve for $\theta^*$ by estimating the gradient of policy performance $J_\gamma(\pi_\theta)$ with respect to policy parameters $\theta$. This is also commonly referred to as the policy gradient theorem $\nabla_\theta J_1(\pi_\theta)$ ~\citep{DBLP:conf/nips/SuttonMSM99} which says
\begin{align}\label{eq:policy_gradient} 
    \nabla_\theta J_\gamma(\pi_\theta) = \mathbb{E}_{(s, a) \sim d^{\pi_\theta}_\gamma}[Q^{\pi_\theta}_{\gamma}(s, a) \cdot \nabla_\theta \log {\pi_\theta}(a | s)], \quad \gamma \in [0, 1].
\end{align}
Here, $d^{\pi_\theta}_\gamma$ is the average state-action discounted stationary distribution, which is defined as the cumulative sum of discounted state-action occupancy across the time horizon.
\begin{subequations}\label{eq:discounted_form}
\begin{align}
    d_\gamma^{\pi_\theta}(s, a) &:= (1-\gamma)\sum_{t=0}^\infty \gamma^t \Prob(s_t = s, a_t = a |s_0 \sim d_0, a_t \sim \pi_\theta(s_t), s_{t+1} \sim \mathcal{P}(\cdot | s_t, a_t), \gamma < 1)   \\
    d_1^{\pi_\theta}(s,a) &:= \lim_{T \rightarrow \infty}\frac{1}{T}\sum_{t=0}^T \Prob(s_t = s, a_t = a |
    s_0 \sim d_0, a_t \sim \pi_\theta(s_t), s_{t+1} \sim \mathcal{P}(\cdot | s_t, a_t), \gamma = 1). 
\end{align}
\end{subequations}
In this paper we make a standard assumption that the Markov chain induced by policy $\pi_{\theta}$ is ergodic. In particular, this implies that  $d_\gamma^{\pi_\theta}(s, a) > 0$ for all state action pairs $(s,a)$ ~\citep{DBLP:books/wi/Puterman2014}.
In scenarios where we are trying to optimize for $J_1(\pi)$, estimating the policy gradient becomes difficult. This is because the Bellman equation cannot be used to estimate Q-function for $\gamma=1$. As a compromise policy gradient for average reward scenarios are instead approximated by calculating the Q-function for a discounting factor $\gamma < 1$, but close to 1,  and using that estimate in the policy gradient equation \ref{eq:policy_gradient}
\begin{subequations}\label{eq:practical_pg}
\begin{align}
    \hat{\nabla}_\theta J_1(\pi_\theta)&= \mathbb{E}_{(s, a) \sim d^{\pi_\theta}_1}[{Q^{\pi_\theta}_1 (s, a)} \cdot \nabla_\theta \log \pi_\theta(a | s)],\\
    &\approx \mathbb{E}_{(s, a) \sim d^{\pi_\theta}_1}[\textcolor{red}{Q^{\pi_\theta}_\gamma (s, a)} \cdot \nabla_\theta \log \pi_\theta(a | s)], \quad \gamma < 1. \label{eq:practical_pg2}
\end{align}
\end{subequations}
In this paper we argue that the policy gradient calculated in this manner induces a significant residual error, which keeps on compounding as the reinforcement learning training proceeds even leading to a sub optimal solution. The following equation, derived in Proposition \ref{prop:res-error} characterizes that error,
\begin{align}\label{eq:res-error}
    \nabla_\theta J_{1}(\pi_\theta) &= \mathbb{E}_{(s, a) \sim d^{\pi_\theta}_1}[ \color{red}{Q^{\pi_\theta}_\gamma (s, a)}  \color{black}{\cdot 
 \nabla_\theta \log \pi_\theta(a | s)} ]+ \underbrace{(1-\gamma) \EE_{(s, a) \sim d_1^{\pi_\theta}}[\nabla_\theta \log d_1^{\pi_\theta}(s) \cdot V^{\pi_\theta}_\gamma(s)]}_{\text{\textbf{Residual Error}}}
\end{align}
\todo[inline]{plots for justifying the significance of the  residual error}
In this paper, we prove that this residual error is significant. What more, we also propose another method to exactly obtain the policy gradient, which we call as the log density gradient. Our estimation of log density gradient utilises average state-action discounted distributional formulation of a reinforcement learning problem which re-states $J_\gamma(\pi)$ as expectation under $d^{\pi_\theta}_\gamma$~\citep{DBLP:conf/nips/NachumCD019, DBLP:conf/icml/UeharaHJ20} as
\begin{align*}
    J_\gamma(\pi) = \mathbb{E}_{(s, a) \sim d^{\pi_\theta}_\gamma}[r(s, a)].
\end{align*}
Under this formulation, policy gradient can similarly be obtained as
% equal the expectation of the product of $\nabla_\theta \log d^{\pi_\theta}_\gamma$ and the reward function
\begin{align}\label{eq:secondform}
    \nabla_\theta J_\gamma(\pi) = \mathbb{E}_{(s, a) \sim d^{\pi_\theta}_\gamma}[\nabla_\theta \log d^{\pi_\theta}_\gamma(s, a) \cdot r(s, a)].
\end{align}
We refer to $\nabla_\theta \log d^{\pi_\theta}_\gamma$ as the log density gradient. A key advantage of log density gradient is that it would allow us to approximate policy gradient for average reward scenarios in a provable manner. In this work, we show that log density gradient can be approximated even under average reward scenarios ($\gamma = 1$). To that end, we first propose a model based approach to approximate the log density gradient for tabular scenarios. Under reversibility assumptions we then show that this log density gradient can also be approximated using TD method. Since reversibility assumption is highly restrictive, we finally propose a min-max version of log density gradient which allows to approximate $\nabla_\theta \log d^{\pi_\theta}$ using empirical samples. Our experimental further show that this manner of policy gradient estimation can potentially make reinforcement learning sample efficient, thus helping them scale. 
\section{Related Work}\label{sec:related_work}
In this section we will discuss existing studies in policy gradient methods including the framework of log density gradient first introduced by \cite{DBLP:journals/neco/MorimuraUYPD10}. We also briefly discuss density ratio learning methods which have been very popular in off-policy evaluation. A short discussion on Temporal Difference (TD) learning methods may be found in Appendix \ref{sec:td_lit}.  
\subsection{Policy Gradient Methods}
Policy gradient methods are a widely studied topic in reinforcement learning. One of the earliest works in this area proposed a closed-form solution for evaluating policy gradients called the policy gradient theorem ~\citep{DBLP:conf/nips/SuttonMSM99}. Initially implementations for policy gradient methods used episodic estimates to update policy parameters~\citep{DBLP:journals/ml/Williams92} and GPOMDP~\citep{pg_baxter}. Unfortunately this way of implementing policy gradient suffered from high variance, thus inhibiting scalability to large problem spaces~\citep{DBLP:journals/corr/SchulmanMLJA15}. 
To address this problem actor-critic methods  approximate the Q-function or advantage function using an additional neural network, which are then used to update the policy~\citep{DBLP:conf/icml/MnihBMGLHSK16, DBLP:journals/corr/SchulmanMLJA15}. Furthermore policy gradient methods were even designed to make them compatible with deterministic policies ~\citep{DBLP:journals/corr/LillicrapHPHETS15, DBLP:conf/icml/SilverLHDWR14}. Recently Trust region methods, such as Trust Region Policy Optimization (TRPO) ~\citet{DBLP:journals/corr/SchulmanLMJA15} and Proximal Policy Optimization (PPO) ~\cite{DBLP:journals/corr/SchulmanWDRK17}, update policies while ensuring monotonic performance improvement. In this work, we argue that although these methods have been successful, their practical implementation still ignores a residual error, correcting for which may increase sample efficiency of policy gradient methods. To the best of our knowledge, Log density gradient has only been discussed in  ~\citet{DBLP:journals/neco/MorimuraUYPD10}. Here, the authors propose a TD method to estimate log density gradient for average reward scenarios by using reversible backward Markov chain. 
In this work, we re-introduce the idea of log density gradient but for a general discounting factor $\gamma \in [0, 1]$, as opposed to the case of $\gamma = 1$ considered in \cite{DBLP:journals/neco/MorimuraUYPD10}. We also show that this method of estimating the log density gradient is unique and tends to converge under linear function approximations. We further propose a min-max version of estimating log density gradient by utilizing Fenchel duality. We also prove that this min-max algorithm converges to a unique solution under linear function approximation assumption. 
\begin{figure}[tb]
    \centering
\includegraphics[width=0.9\textwidth]{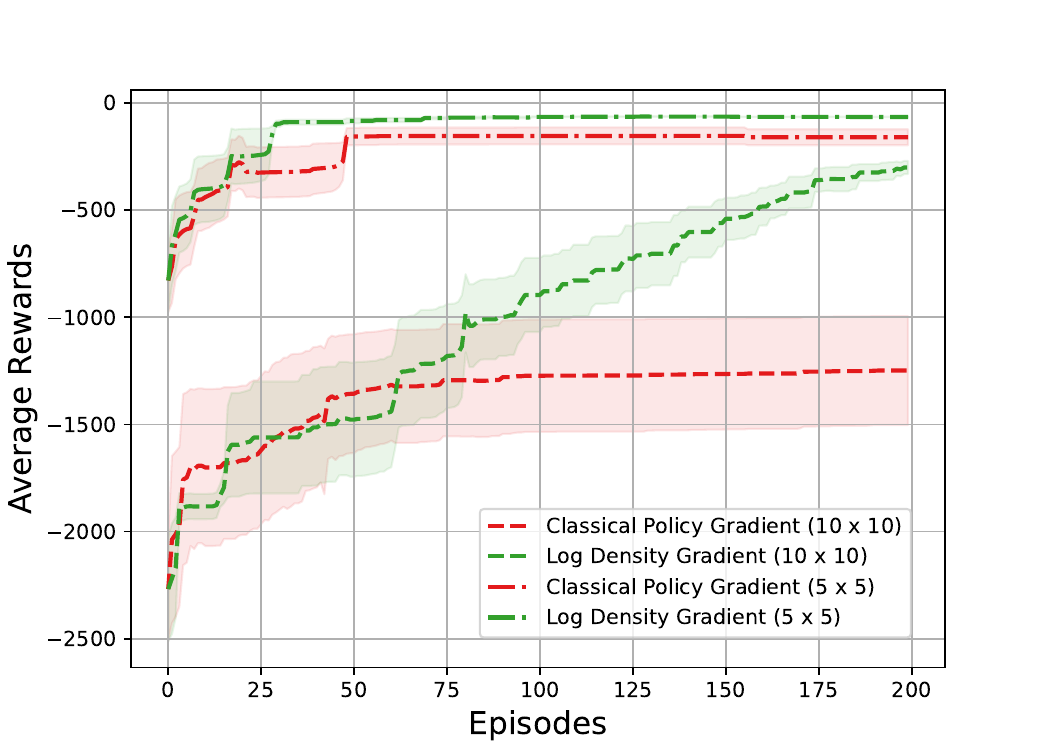}
    \caption{For the average reward scenario, performance of classical policy gradient (red) algorithm as compared to log density gradient (green) algorithm over a $n \times n$ gridworld environment, for $n=5,10$. We observe that log density gradient algorithm consistently converges to better policy performance. Theoretical calculated solutions are used for implementation. }
    \label{fig:motivation}
\end{figure}
%aim to estimate the policy gradient by using the average state-action discounted distribution formulation for reinforcement learning. Under this formulation, we argue that the policy gradient can also be estimated as the gradient of the log state-action discounted distribution. To the best of our knowledge, the idea of log density gradient has only been discussed in ~\cite{DBLP:journals/neco/MorimuraUYPD10}. Here, the authors propose a TD method to estimate the log density gradient for the average reward scenario by using the reversible-backward Markov chain. 

\subsection{Density Ratio Learning}
Off-policy evaluation estimates the performance of a target policy $\pi$ using an offline dataset generated by a behavior policy $\mu$~\citep{DBLP:journals/corr/abs-1911-06854, DBLP:conf/icra/KatdareLC22}. This is done by estimating the average state-action density ratio $\frac{d^{\pi}_\gamma}{d^\mu}$, which allows approximation of the target policy's performance. In this work, we are primarily interested in the DICE class of off-policy evaluation algorithms ~\citep{DBLP:journals/corr/abs-2001-11113, DBLP:conf/nips/NachumCD019, DBLP:journals/corr/abs-2002-09072}. These algorithms typically approximate the divergence (for some $f-$divergence of their choice) between the two distributions $d^{\pi}$ and $d^{\mu}$ in their convex dual form, eliminating the need to obtain samples from $d^{\pi}$, which results in a min-max form optimization problem. Inspired by the DICE class of algorithms we too propose a min-max form of estimating the log density gradient. We show that such a method of estimating log density gradient converges to the true policy under linear function approximations assumptions. We also show that the sample complexity of such an estimator is of the order $(n^{-1/2})$, with $n$ being the number of on-policy samples. 
%Applications of such min-max form for density ratio learning are not limited to off-policy evaluation and are now being used for offline reinforcement learning ~\citet{DBLP:journals/corr/abs-2106-10783, DBLP:journals/corr/abs-1912-02074} and off-policy policy gradient ~\citet{DBLP:conf/uai/LiuSAB19}. 
\section{Log Density Gradient}\label{sec:our_method}
In this section, we introduce log density gradient and further show that classical policy gradient typically ignores a residual error which may be significant. We then propose a TD version of our log density gradient method that estimates the policy gradient without a need to account for this error. 
\subsection{Model Based Log Density Gradient}
The log density gradient method attempts to estimate the gradient of log of average state-action discounted stationery distribution $d^{\pi_\theta}_\gamma$. We start by observing that $d_\gamma^{\pi_\theta}$ satisfies an identity called the Bellman flow equation ~\cite{DBLP:conf/nips/LiuLTZ18}.
\begin{lemma}\label{le:bellman_conservation}
The average state-action density distribution satisfies the following identity for all $\gamma \in [0, 1]$,
\begin{align}\label{eq:gradient_identity}
    d_\gamma^{\pi_\theta}(s') &= (1-\gamma)d_0(s') + \gamma \sum_{s, a} d_\gamma^{\pi_\theta}(s, a) \mathcal{P}(s' | s, a) 
\end{align}
\end{lemma}
We prove this result in the appendix \ref{proof:bellman_conservation}. Note that this equation is similar to the Bellman equation but in a backward manner, which means, we need samples from the backward conditional distribution (defined rigorously in Section \ref{sec:tdldg}). 
It can also be understood as a form of flow conservation, wherein the flow out of state $s'$ (LHS) would be equal to the flow into $s'$ (RHS). 
%\hl{Bellman flow equation guarantees that for a given an policy} $\pi_\theta$, \hl{the solution the Bellman flow equation \ref{eq:gradient_identity} exists and is unique} \cite{DBLP:journals/corr/abs-2001-11113}. 
% \hl{Bellman flow equation guarantees that for a given an policy $\pi_\theta$,} \hl{the solution the Bellman flow equation {\ref{eq:gradient_identity}} }  \hl{exists and is unique} \cite{DBLP:journals/corr/abs-2001-11113}. 

\begin{lemma}\label{le:bellman_conservation_uniqueness}
For all $\gamma \in [0, 1]$ the solution to the following optimisation is unique and equal to $d_\gamma^{\pi_\theta}$.
\begin{align}\label{eq:bellman_unique_1}
    \argmin_{w: \S \rightarrow \mathbb{R}} &\sum_{s'} \left( w(s') - (1-\gamma)d_0(s') + \gamma \sum_{s, a} w(s) \pi_\theta(a | s) \P(s' | s, a)\right)^2 + \frac{\lambda}{2} (\sum_{s} w(s) - 1)^2 
\end{align} 
\end{lemma}
% \todo[inline]{why do we need to restate the assumptions in the lemma?}
Detailed proof can be found in the appendix \ref{proof:bellman_conservation_uniqueness}. Intuitively speaking, the term $\frac{\lambda}{2} (\sum_{s} w(s) - 1)^2 $ is redundant for $\gamma <1$. It becomes useful for $\gamma = 1$ to ensure uniqueness. Although it is hard to estimate a closed form solution for $\nabla_\theta \log d^{\pi_\theta}_\gamma$, it is still possible to estimate it numerically. Using \eqref{eq:gradient_identity} we can similarly calculate an equivalent identity for $\nabla_{\theta}\log d_{\gamma}^{\pi_{\theta}}$
\begin{align}\label{eq:ldg_identity}
    % \nabla_\theta d_\gamma^{\pi_\theta} (s') &= \gamma \sum_{s \in \S, a \in \A} \nabla_\theta d_\gamma^{\pi_\theta}(s, a) \P(s' | s, a) \\
    d_\gamma^{\pi_\theta}(s') \nabla_\theta \log d_\gamma^{\pi_\theta}(s')  &= \gamma \sum_{s, a}d_\gamma^{\pi_\theta}(s, a) \mathcal{P}(s' | s, a) \nabla_\theta \log d_\gamma^{\pi_\theta}(s, a).
\end{align}
% \todo[inline]{is the first equation here necessary?}
% We go from the first line to the second line using the log-derivative trick. 
% Since, $d_\gamma^{\pi_\theta}(s, a) = d_\gamma^{\pi_\theta}(s) \pi_\theta (a | s)$, we can similarly conclude that 
% $\nabla_\theta \log d_\gamma^{\pi_\theta}(s', a') = \nabla_\theta \log d_\gamma^{\pi_\theta}(s') + \nabla_\theta \log {\pi_\theta}(a' | s')$. Substituting and multiplying both sides by $\pi_\theta(a' | s')$, we get the following 
Multiplying both sides by $\pi_\theta(a' | s')$ and recalling that $d_\gamma^{\pi_\theta}(s, a) = d_\gamma^{\pi_\theta}(s) \pi_\theta (a | s)$, we obtain
\begin{align}\label{eq:ldg_condition}
% \begin{split}
    &d_\gamma^{\pi_\theta}(s', a') (\nabla_\theta \log d_\gamma^{\pi_\theta}(s', a') - \nabla_\theta \log \pi(a' | s'))  = \gamma \sum_{s, a}d_\gamma^{\pi_\theta}(s, a) \nabla_\theta \log d_\gamma^{\pi_\theta}(s, a) \mathcal{P}(s' | s, a) \pi_\theta(a' | s')
    % \end{split}
\end{align}
% \todo[inline]{unclear what is done in above eqn?}
Similar to Lemma \ref{le:bellman_conservation_uniqueness} we can estimate $\nabla_\theta \log d_\gamma^{\pi_\theta}$ by solving for the equation \ref{eq:ldg_condition} above. To that end, we propose the following optimization (where $\lambda > 0$ is a fixed regularizer) %to estimate $\nabla_\theta \log d_\gamma^{\pi_\theta}$; 
% \begin{align}\label{eq: opt_ldg}
% \begin{split}
% \nabla_\theta \log d_\gamma^{\pi_\theta} \simeq &\min_{w: \mathcal{S} \times \mathcal{A} \rightarrow \mathbb{R}^{n}} \EE_{(s', a') \sim d_\gamma^{\pi_\theta}} \|d_\gamma^{\pi_\theta}(s', a') (w(s', a') 
%     - \nabla_\theta \log \pi_\theta(a' |s'))\\
%     &- \gamma \sum_{s, a}d_\gamma^{\pi_\theta}(s, a) \mathcal{P}(s' | s, a) \pi_\theta(a'|s') w(s, a) \| ^2  + \frac{\lambda}{2} \| E_{(s, a) \sim d^{\pi_\theta}_\gamma}[w(s', a')] \|^2
%     \end{split}
% \end{align}
\begin{align}\label{eq: opt_ldg}
&\min_{w: \mathcal{S} \times \mathcal{A} \rightarrow \mathbb{R}^{n}} 
% \bigg\{ \EE_{(s', a') \sim d_\gamma^{\pi_\theta}} \Big\|d_\gamma^{\pi_\theta}(s', a') (w(s', a') - \nabla_\theta \log \pi_\theta(a' |s'))  \nonumber \\
%     & \qquad - \gamma \sum_{s, a}d_\gamma^{\pi_\theta}(s, a) \mathcal{P}(s' | s, a) \pi_\theta(a'|s') w(s, a) \Big\| ^2   + \frac{\lambda}{2} \Big\| E_{(s, a) \sim d^{\pi_\theta}_\gamma}[w(s', a')] \Big\|^2 \bigg\}.
\bigg\{ \EE_{(s', a') \sim d_\gamma^{\pi_\theta}} \| \nu(s',a') \| ^2   + \frac{\lambda}{2} \Big\| E_{(s, a) \sim d^{\pi_\theta}_\gamma}[w(s', a')] \Big\|^2 \bigg\} \\
    &\nu(s',a') := d_\gamma^{\pi_\theta}(s', a') (w(s', a') - \nabla_\theta \log \pi_\theta(a' |s')) - \gamma \sum_{s, a}d_\gamma^{\pi_\theta}(s, a) \mathcal{P}(s' | s, a) \pi_\theta(a'|s') w(s, a) \nonumber
\end{align}
% \todo[inline]{check definition of $\delta$ once}
% \begin{lemma}\label{le:uniquenss}
% Given that state-action spaces are finite $\S, \A$ and that, the policy $\pi_\theta$ induces a distribution $d^{\pi_\theta}_\gamma(s, a) > 0, \forall (s, a) \in \S \times \A$. The solution to the following set of equations is unique and equal to $\nabla_\theta \log d_\gamma^{\pi_\theta}$ for all $ \gamma \in [0, 1]$,
% \begin{align}\label{eq:bellman_unique}
%     w^* &=  \min_{w: \mathcal{S} \times \mathcal{A} \rightarrow \mathbb{R}^{n}} \EE_{(s', a') \sim d_\gamma^{\pi_\theta}} \|d_\gamma^{\pi_\theta}(s', a') (w(s', a') 
%     - \nabla_\theta \log \pi_\theta(a' |s'))\\
%     &- \gamma \sum_{s, a}d_\gamma^{\pi_\theta}(s, a) \mathcal{P}(s' | s, a) \pi_\theta(a'|s') w(s, a) \| ^2  + \frac{\lambda}{2} \| E_{(s, a) \sim d^{\pi_\theta}_\gamma}[w(s', a')] \|^2
% \end{align}
% \end{lemma}
% \begin{assumption}\label{ass:ergodic}
%     The Markov chain induced by policy $\pi_{\theta}$ is ergodic. In particular, it implies that  $d_\gamma^{\pi_\theta}(s, a) > 0$ for all state action pairs $(s,a)$ ~\citep{DBLP:books/wi/Puterman2014}. 
% \end{assumption}
\begin{lemma}\label{le:uniquenss}
% Under assumption \ref{ass:ergodic} 
The solution to \eqref{eq: opt_ldg} is unique and equal to $\nabla_\theta \log d_\gamma^{\pi_\theta}$ for all $ \gamma \in [0, 1]$.
\end{lemma}
% \todo[inline]{1. can equation just above lemma be erased? 2. does lemma mean to say that $w^* = \nabla \log d^{\pi}$? if yes then just write $\nabla \log .. = w^*$ where $w^*$ is defined as}
We describe the proof for this Lemma in the appendix \ref{proof:uniquenss}. It is worth noting that, the constraint $\frac{\lambda}{2} \| E_{(s, a) \sim d^{\pi_\theta}_\gamma}[w(s', a')] \|^2$ is only useful for $\gamma = 1$. The proof follows from the fact that the solution to the equation \ref{eq: opt_ldg} can be written in a linear form $A\cdot w = b$ and showing that A is invertible. It is worth reiterating that, once we have an estimate $\nabla_\theta \log d_\gamma^{\pi_\theta}$, we can use this estimate to approximate the policy gradient using equation \ref{eq:secondform}. We will now recall two important properties of log density gradient. 
\begin{proposition}\label{prop:consistency}
The policy gradient method as mentioned in equation \ref{eq:secondform} is exactly equal to the classical policy gradient ~\citep{DBLP:conf/nips/SuttonMSM99} as mentioned in equation \ref{eq:policy_gradient}.
\end{proposition}
Detailed proof for this proposition can be found in Appendix \ref{sec:pg_consistency}. In essence this means that log density gradient calculates the same policy gradient but using a different formulation. We show next that this formulation allows us to correct for the residual error in policy approximation which is typically ignored in many actor-critic implementations of policy gradient methods. 
% \todo[inline]{1. add a short into before proposition like "we study two properties of ..." }
% \begin{proposition}\label{prop:res-error}
% Actor Critic implementation of policy gradient methods first approximate the Q-function using a discounting factor $\gamma$ strictly less than 1 and then approximate the policy gradient $\nabla_\theta J(\pi_\theta)$ using equation \ref{eq:practical_pg}. This leads to a residual error in policy gradient approximation which is equal to $(1-\gamma) \mathbb{E}_{s \sim d^{\pi_\theta}_1}[\nabla_\theta \log d^{\pi_\theta}_1(s) \cdot V^{\pi_\theta}_\gamma(s)]$.
% \end{proposition}
\begin{proposition}\label{prop:res-error}
    The following identity, stated in \eqref{eq:res-error}, is true
    \begin{align*}
    \nabla_\theta J_{1}(\pi_\theta) &= \mathbb{E}_{(s, a) \sim d^{\pi_\theta}_1}[ \color{red}{Q^{\pi_\theta}_\gamma (s, a)}  \color{black}{\cdot 
 \nabla_\theta \log \pi_\theta(a | s)} ]+ \underbrace{(1-\gamma) \EE_{(s, a) \sim d_1^{\pi_\theta}}[\nabla_\theta \log d_1^{\pi_\theta}(s) \cdot V^{\pi_\theta}_\gamma(s)]}_{\text{\textbf{Residual Error}}}
\end{align*}
\end{proposition}
\begin{proof}
From the definition of log density gradient \eqref{eq:secondform} we have 
    $\nabla_\theta J_1(\pi_\theta) = \EE_{(s, a) \sim d_1^{\pi_\theta}}[\nabla_\theta \log d_1^{\pi_\theta}(s, a) \cdot r(s, a)]$. 
Let $\gamma < 1$, and we use the Bellman equation \ref{eq:bellman} to obtain
\begin{subequations}
    \begin{align}
    \nabla_\theta J_1(\pi_\theta) &= \EE_{(s, a) \sim d_1^{\pi_\theta}}[\nabla_\theta \log d_1^{\pi_\theta}(s, a) \cdot (Q^{\pi_\theta}_{\gamma}(s, a)-\gamma \EE_{s' \sim \mathcal{P}(\cdot | s, a)}[V^{\pi_\theta}_{\gamma}(s')])] \label{eq:residue_1}\\
    &= \EE_{(s, a) \sim d_1^{\pi_\theta}}[(\nabla_\theta \log d_1^{\pi_\theta}(s) + \nabla_\theta \log {\pi_\theta}(a | s)) \cdot (Q^{\pi_\theta}_{\gamma}(s, a)-\gamma \EE_{s' \sim \mathcal{P}(\cdot | s, a)}[V^{\pi_\theta}_{\gamma}(s')])] \label{eq:residue_2} \\
    &= \EE_{(s, a) \sim d_1^{\pi_\theta}}[ Q^{\pi_\theta}_{\gamma}(s, a) \cdot \nabla_\theta \log \pi_\theta(a | s)] + (1-\gamma) \EE_{s \sim d_1^{\pi_\theta}}[\nabla_\theta \log d_1^{\pi_\theta}(s, a) \cdot V^{\pi_\theta}(s)] \label{eq:residue_3} 
\end{align}
Here, we go from equation \ref{eq:residue_1} to \ref{eq:residue_2} by utilizing $\nabla_\theta \log d^{\pi_\theta}_\gamma(s, a) = \nabla_\theta \log d^{\pi_\theta}_\gamma (s) + \nabla_\theta \log \pi_\theta(a | s)$. We finally go from \ref{eq:residue_2} to \ref{eq:residue_3} by using a key identity of log density gradient equation \ref{eq:ldg_condition}. 

using identity for log density gradient as shown in the equation \ref{eq:ldg_identity}. 
Note that the first term in equation \ref{eq:residue_2} is the practical instantiation of classical policy gradient theorem, while the second term $(1-\gamma) \mathbb{E}_{s \sim d^{\pi_\theta}_1}[\nabla_\theta \log d^{\pi_\theta}_1(s) \cdot V^{\pi_\theta}_\gamma(s)]$ being the residual error. 
This completes the proof. 
\end{subequations}
\end{proof}
Actor-Critic implementation of policy gradient methods first approximate the Q-function using a discounting factor $\gamma$ strictly less than 1 and then approximate the policy gradient $\nabla_\theta J_1(\pi_\theta)$ using equation \ref{eq:practical_pg2}. This leads to a residual error in policy gradient approximation as shown in Proposition \ref{prop:res-error}.  We believe that correcting for this gradient estimation can potentially make policy gradient algorithms sample efficient and scalable to complex problems. 

Although the solution to equation \eqref{eq: opt_ldg} is unique, solving it requires access the transition matrix $\P$ which is impractical for complex environments. Additionally, the log density gradient estimation also becomes computationally infeasible with exponential scaling of state-action. We thus propose a temporal difference approach to estimate log density gradient from empirical data next . 
\subsection{Temporal Difference Log Density Gradient}
\label{sec:tdldg}
To propose an update equation for temporal difference (TD) method, we first begin with re-arranging few terms in equation \ref{eq:ldg_condition} for log density gradient. 
\begin{align*}
    \nabla_\theta \log d^{\pi_\theta}_\gamma(s', a') = \nabla_\theta \log \pi_\theta(a' | s') +  \gamma\sum_{s, a} \frac{d^{\pi_\theta}_\gamma(s, a) \P(s' | s, a) \pi_\theta(a' | s')}{d^{\pi_\theta}_\gamma(s' ,a')}  \nabla_\theta \log d^{\pi_\theta}_\gamma(s, a)
\end{align*}
We define the backward distribution of $(s,a)$ given $(s',a')$ as
% The first term in the summation is
\begin{align*}
    \mathcal{P}_b(s,a|s',a') \coloneqq \frac{d^{\pi_\theta}_\gamma(s, a) \P(s' | s, a) \pi_\theta(a' | s')}{d^{\pi_\theta}_\gamma(s' ,a')} = \frac{d^{\pi_\theta}_\gamma(s, a) \P^{\pi_{\theta}}(s',a' | s, a) }{d^{\pi_\theta}_\gamma(s' ,a')}
\end{align*}
which is a consequence of Bayes' rule. The summation therefore becomes an expectation under $\mathcal{P}_b$. The log density gradient is therefore said to follow a backward recursion and it requires samples from backward conditional probability $\mathcal{P}_b$ to estimate log density gradient\footnote{Although for $\gamma=1$ we can use samples from $\mathcal{P}$ as well \citep{DBLP:journals/neco/MorimuraUYPD10} }.  We first generalize algorithm of  \cite{DBLP:journals/neco/MorimuraUYPD10}, who do it only for $\gamma = 1$, to estimate log density gradient $w$ for all discounting factor $\gamma \in [0, 1]$ in form of a temporal difference (TD(0)) method where our update equation is
\begin{align}\label{eq: td_ori}
    w(s', a') \leftarrow w(s', a') + \alpha[\gamma w(s, a) + g(s', a') - w(s', a')]
\end{align}
with $(s', a') \sim d_\gamma^{\pi_\theta}, (s,a) \sim \mathcal{P}_b(\cdot | s',a') $
% 
% with $(s', a') \sim d_\gamma^{\pi_\theta}, (s, a) \sim \mathcal{P}(\cdot | s',a', d^{\pi_\theta}_\gamma)$ 
and $g(s', a') := \nabla_\theta \log \pi_\theta(a'| s')$. 
% It is worth re-iterating that since the log density gradient follows a backward recursion (equation \ref{eq:ldg_condition}) it ideally requires samples from backward conditional probability $\mathcal{P}_b$ to estimate log density gradient\footnote{Although that is not the case for $\gamma=1$ ~\cite{DBLP:journals/neco/MorimuraUYPD10} }. 
Define operator $Y_\gamma$ to capture the behaviour of update rule \eqref{eq: td_ori} after taking expectation,
\begin{align*}
    (Y_\gamma \cdot w)(s', a') := \gamma \mathbb{E}_{(s, a) \sim d_\gamma^{\pi_\theta}}[ w(s, a)| (s', a')] +  g(s', a').
\end{align*}
% \todo[inline]{proof for behaviour under expectation\\
% \textbf{Answer:} So the key idea is that, the equation 14 can be written in form of its expectation. Basically the motivation for the same comes from the paper "Off-Policy Deep Reinforcement Learning by Bootstrapping the Covariate Shift" (equation 4 and equation 5). Can remove if you want to. $Y_\gamma$ is then written in form of an operator which is applied to your current estimate $c$ to give you the new estimate. Here, $ g(s, a) = \nabla_\theta \log \pi(a | s)$}
We can write this in matrix form as follows, 
\begin{align}\label{eq:contraction_operator}
    Y_\gamma \cdot W = \gamma D_{\pi_\theta}^{-1} \P_{\pi_\theta}^\top \D W + G
\end{align}
where, $W \in \mathbb{R}^{|\S| \cdot | \A | \times n}$ is the matrix with every row corresponding to $w(s,a)$ for each state-action pair $(s,a)$. Similarly, $G \in \mathbb{R}^{|\S| \cdot | \A | \times n}$ has its rows as $\nabla_\theta \log \pi_\theta$ for each state-action pair. Let $\P_{\pi_\theta}, \D \in \mathbb{R}^{|\S| \cdot | \A | \times |\S| \cdot | \A |}$ where $ (\P_{\pi_\theta})_{((s,a),(s',a'))} = \Prob^{\pi_{\theta}}(s',a'|s,a)$ and $\D$ is a diagonal matrix whose every element correspond to $d_\gamma^{\pi_\theta}$ for each  state-action pair. We use this matrix form for the operator $Y_{\gamma}$ in the proof of the following lemma. 
% \todo[inline]{1. give reference for classic TD(0) 2. erased the definition environment, erased assumption on gamma\\
% \textbf{Answer:} By TD(0) I mean we propose a TD(0) algorithm to estimate log density gradient using data via temporal difference algorithm}
% \begin{definition}
% Consider the set $C = \{c(s, a) \in \mathbb{R}^{n} : \forall s, a \in \mathcal{S} \times \mathcal{A}\}$. Given that the discount factor $\gamma \in [0, 1]$ and $(s, a) \sim d_\gamma^{\pi_\theta}, s' \sim \mathcal{P}(\cdot | s, a), a' \sim \pi_\theta(a' | s')$. We propose a TD(0) update rule for estimating the policy gradient.  
% \begin{align}\label{eq: td_ori}
%     c(s', a') \leftarrow c(s', a') + \alpha[\gamma c(s, a) + g(s', a') - c(s', a')]
% \end{align}
% \end{definition}
\begin{lemma}\label{le:fixed_point}
   Let $w_0 \in \Delta(\S, \A )$ be an arbitrary initial guess.
   %and suppose that  assumption \ref{ass:ergodic} holds. 
   Let $w_k = Y_{\gamma} \cdot w_{k-1}$ for all natural numbers $k \ge 1$.  For $\gamma \in [0,1)$, the operator $Y_{\gamma}$ is a contraction, and $\{w_k\}_{k \ge 0}$ converges to a unique fixed point $\nabla_\theta \log d^{\pi_\theta}_\gamma$. 
   %For $\gamma = 1$, $\{w_k\}_{k \ge 0}$ converges to a fixed point which is of the form $\nabla_\theta \log d_{1}^{\pi_\theta} + \kappa$, for some constant $\kappa$. 
\end{lemma}
% \todo[inline]{add ergodic ssumption to preliminaries as it is also used later}
Detailed proof of Lemma \ref{le:fixed_point} can be found in Appendix \ref{sec:fixed_point}. 
% The log density gradient estimation for $\gamma = 1$ is always up to a constant $\kappa$. This is because the operator $Y_1$ fails to be a contraction and thus the TD update converges to a subspace rather than a fixed point. 
% Since $\mathbb{E}_{d^{\pi_\theta}_1}[\nabla_\theta \log d^{\pi_\theta}_1] = 0$, we correct for this as follows
% \begin{equation*}%\label{Le:normal_ldg}
% \nabla_\theta \log d_1^{\pi_\theta} = \lim_{k \to \infty}(Y_1^k W_0) - \EE_{(s, a) \sim d_1^{\pi_\theta}}[\lim_{k \to \infty}(Y_1^k W_0)].
% \end{equation*}
Extension of Lemma \ref{le:fixed_point} to linear function approximation, and proof of convergence for the same, can be found in the Appendix \ref{sec:linear_td_details}. Although TD methods are known to converge, they still suffer from two problems. One, the access to samples from backward conditional probability. Two, scalability to large problem spaces. We attempt to solve both of these problems in the next section where we propose a min-max optimization procedure for estimating the log density gradient. 

\section{Min-Max Log Density Gradient}\label{sec:min_max}
In this section, we propose another approach for evaluating the log density gradient which uses the min-max form for optimization which does not need samples from the backward distribution. Min-max optimizations also allow us to use a large variety of function classes like neural networks to estimate log density gradient. 

Let us return to the loss function that we initially propose in  \eqref{eq: opt_ldg}. Classical machine learning algorithms typically require loss function in the form of an expectation under a distribution. These kind of algorithms are often called as Empiricial Risk Minimization, which allows us to approximate the loss using samples from that distribution. Consider a modified form of the optimization  proposed in \eqref{eq: opt_ldg} (the modification is that $\delta(s',a')$ is divided by $d_\gamma^{\pi_\theta}(s', a')$ where the ergodicity assumption ensures this operation is well defined),
% \todo[inline]{what modification did you do exactly\\
% \textbf{Answer: }Look at the denominator term $d^{\pi_\theta}$ below delta. That is the modification}
\begin{align}
    %&\nabla_\theta \log d_\gamma^{\pi_\theta} \simeq 
    &\argmin_{w \in \mathcal{S} \times \mathcal{A} \rightarrow \mathbb{R}^n} \mathbb{E}_{(s', a') \sim d_\gamma^{\pi_\theta}} \left[\left\|\frac{\nu(s', a')}{d_\gamma^{\pi_\theta}(s', a')}\right\|^2\right]+ \frac{\lambda}{2}\|\mathbb{E}_{(s, a) \sim d_\gamma^{\pi_\theta}}[w(s, a)]\|^2 \label{eq:opt_ldg_minmax} \\
    &\nu(s', a') := d_\gamma^{\pi_\theta}(s', a') (w(s', a') - \nabla_\theta \log \pi_\theta(a' |s'))- \gamma \sum_{s, a}d_\gamma^{\pi_\theta}(s, a) \mathcal{P}(s' | s, a) \pi(a'|s') w(s, a) \nonumber
\end{align}
%We observe that the minimization is now over the function class $\mathcal{W}$ as opposed to all functions of $w$. 
%\todo[inline]{need proof of following lemma}
% \begin{lemma}\label{le:bellman_conservation_uniqueness}
% The solution to \eqref{eq:opt_ldg_minmax} is unique and equal to $\nabla_\theta \log d_\gamma^{\pi_\theta}$ for all $ \gamma \in [0, 1]$.
% \end{lemma}Detailed proof is written in Appendix \ref{sec:proof_of_uniqueness}. 
The denominator term $d^{\pi_\theta}_\gamma$ is added to simplify the final optimization, which we shall see soon. We also add the term $\frac{\lambda}{2}\|\mathbb{E}_{(s, a) \sim d_\gamma^{\pi_\theta}}[w(s, a)]\|^2$ to satisfy one of the properties of the gradient of log density $\mathbb{E}_{(s, a) \sim d_\gamma^{\pi_\theta}}[\nabla_\theta \log d_\gamma^{\pi_\theta}(s, a)] = 0$.

It is worth noting that \eqref{eq:opt_ldg_minmax} is just a re-weighting of equation \ref{eq: opt_ldg} with the $\frac{1}{d^{\pi_\theta}_\gamma(s, a)}$. This implies that the optimal solution for the both the equation is the same because the minimum value for both the optimization can only be reached when $w(s, a) = \nabla_\theta \log d^{\pi_\theta}_\gamma(s, a)$.  By exploiting the Fenchel-duality, we can re-write this optimization in the minimax form~\citep{DBLP:books/degruyter/Rockafellar70, DBLP:journals/corr/abs-2001-11113}. 
\begin{align}
&\arg \min_{w: \mathcal{S} \times \mathcal{A} \rightarrow \mathbb{R}^d} \max_{f: \mathcal{S} \times \mathcal{A} \rightarrow \mathbb{R}^d, \tau \in \mathbb{R}^{d}} L_\gamma(w, f, \tau) := \bigg\{ \mathbb{E}_{(s', a') \sim d_\gamma^{\pi_\theta}}[f(s', a') \cdot w(s', a')] \nonumber  \\
&- \mathbb{E}_{(s', a') \sim d_\gamma^{\pi_\theta}}[f(s', a') \cdot \nabla_\theta \log \pi_\theta(a' | s')] - \gamma \mathbb{E}_{(s, a) \sim d_\gamma^{\pi_\theta}}[\mathbb{E}_{s' \sim \mathcal{P}(\cdot | s, a), a' \sim \pi_\theta(\cdot | s')}[f(s', a')] \cdot w(s, a)] \nonumber \\
&-\frac{1}{2} \mathbb{E}_{(s, a) \sim d_\gamma^{\pi_\theta}}[\|f(s', a')\|^2] + \lambda (\tau \cdot \mathbb{E}_{(s, a) \sim d_\gamma^{\pi_\theta}}[w(s, a)] - \frac{1}{2} \|\tau\|^2) \bigg\} \label{eq:opt_ldg_minmax2}
\end{align}
%\todo[inline]{1. let us go through the above together term by term and see how to simplify it 2. what is the solution of this problem? need a theorem for that. }
% We will now prove the convergence of this min-max form under linear function class assumption that is, we let $w(s, a) = \alpha^T \Phi(s, a), f = \beta^T \Phi(s, a)$, where, $\Phi(s, a) \in \mathbb{R}^{d \times 1}$ is the feature vector, and $\alpha, \beta \in \mathbb{R}^{d \times n}$ are the parameters that need to be learnt. Since the function class is linear, we can calculate $w$ for all state-action spaces in a succinct matrix representation. $W = \alpha^T \Psi,  F = \beta^T \Psi$. Here $\Psi$ is the row concatenation of features across all state-action space. Given $\varepsilon_t$ is the learning rate at time $t$, we update our learning parameters as follows, 
In many cases searching over all set of functions is not possible, hence we search over tractable function classes $\mathcal{W},\mathcal{F}$ and the aim is to approximate
\begin{align*}
%\label{eq:minimax}
    \nabla_\theta \log d^{\pi_\theta}_\gamma \approx \arg \min_{w \in \mathcal{W}} \max_{f \in \mathcal{F}, \tau \in \mathbb{R}^n} L_\gamma(w, f, \tau).
\end{align*}
Such a practical consideration allows us to use different types of function approximators like linear function approximation, neural networks, and reproducible kernel Hilbert spaces (RKHS). 
We will now provide an update rule to solve \eqref{eq:opt_ldg_minmax2} under linear function approximation. For that we choose a feature map $\Phi: \mathcal{S} \times \mathcal{A} \to \mathbb{R}^{d}$  and parameters $\alpha, \beta \in \mathbb{R}^{d \times n}$ 
% such that $w(s, a) = \alpha^T\Phi(s, a), f(s, a) = \beta^T \Phi(s, a)$ 
% are the linear parameters 
that need to be learnt, 
so that we can approximate the optima of \eqref{eq:opt_ldg_minmax2},  $w^*(s, a)$ and $f^*(s,a)$ with  $\alpha^T \Phi(s, a)$, and $\beta^T \Phi(s, a)$ respectively, for each state action pair $(s,a)$. %Using the minimax optimization as shown in the equation \ref{eq:minimax}
The update rule is
\begin{subequations} \label{eq:update_equation}
    \begin{align}
&\delta_t = \Phi_t \Phi_t^T - \gamma \Phi_t (\Phi_t')^T  
 \\
    &\alpha_{t+1}^T = \alpha_t^T - \varepsilon_t (\beta^T \delta_t + \lambda (\tau \Phi_t^T ) )  \\ 
    &\beta_{t+1}^T = \beta_t^T + \varepsilon_t( \alpha_t^T \delta_t - g_t\Phi_t^T- \beta_t^T \Phi_t \Phi_t^T)  \\
    &\tau_{t+1} = \tau_t + \varepsilon_t(\lambda (\alpha_t^T \Phi_t  - \tau_t)) 
\end{align}
\end{subequations}    
where, $\Phi_t := \Phi(s_t, a_t)$ is the feature encountered at time $t$, $g_t := \nabla_\theta \log \pi_\theta (a_t | s_t)$, and $\Phi_t' := \Phi_t(s_t', a_t')$ for $(s_t, a_t) \sim d^{\pi_\theta}_\gamma, s_t' \sim \mathcal{P}(\cdot | s_t, a_t), a_t' \sim \pi_\theta(a_t' | s_t')$. 
% We want to approximate $w(s, a)$ with  $\alpha^T \Phi(s, a)$, and $ f (s,a)$ with $\beta^T \Phi(s, a)$ for each state action pair $(s,a)$. 
We first re-write the updates in equation \ref{eq:update_equation} in form of $d_t = [\alpha_t, \beta_t, \tau_t^T ]$ so that the updates can be written in matrix form $d_{t+1} = d_{t} + \varepsilon_t (G_{t+1} d_t + h_{t+1})$, where, $G_{t+1}, h_{t+1}$ are as follows, 
\begin{align*}
    G_{t+1} & := \begin{bmatrix}
    0 & -A_t & -\lambda \Phi_t  \\ A_t & -C_t & 0 \\
    \lambda \Phi_t^T & 0 & -\lambda 
    \end{bmatrix}, \quad h_{t+1} := \begin{bmatrix}
    0 \\
    -B_t \\
    0
    \end{bmatrix}
\end{align*}
and $A_t := (\Phi_t \Phi_t^T - \gamma \Phi_t (\Phi_t')^T), B_t := \Phi g_t^T, C_t := \Phi_t \Phi_t^T$. We can calculate the expectation for each of these matrices as follows, 
\begin{align*}
    G := \mathbb{E}_{p}[G_{t+1}] = \begin{bmatrix}
        0 & -A & -\lambda  \Psi\D e \\
        A & C & 0 \\
        \lambda e^T \D \Psi^T  & 0 & -\lambda
    \end{bmatrix}, \quad  h := \mathbb{E}_{(s, a) \sim d_\gamma^{\pi_\theta}}[h_{t+1}] = \begin{bmatrix}
        0 \\
        -B \\
        0
    \end{bmatrix}
\end{align*}
Here, each column of $\Psi \in \mathbb{R}^{|\mathcal{S}|\cdot |\mathcal{A}| \times n}$ is the feature vector $\Phi(s,a)$, for each $(s,a) \in \mathcal{S} \times \mathcal{A}$ and $e \in \mathbb{R}^n$ is a vector of 1's at every element. We can similarly write $A = \Psi\D(I - \gamma P_{\pi_\theta}) \Psi^T, B = \Psi \D G^T, C = \Psi \D \Psi^T$ and $E_p[\cdot] := \E_{(s, a) \sim d_\gamma^{\pi_\theta}, 
    s' \sim \mathcal{P}(\cdot | s, a), a' \sim \pi_\theta(\cdot | s')}[\cdot]$. We can now prove the convergence of linear function approximation under the following key assumptions. 
    % \todo[inline]{need to discuss definition of $G,h$}
\begin{assumption}\label{ass:3}
   \begin{enumerate}
       \item The matrix $\Psi$ has linearly independent columns. 
       \item The matrix $A$ is non-singular or the regularizer $\lambda > 0$.
       \item The feature matrix $\Phi$ has uniformly bounded second moments.
   \end{enumerate}
\end{assumption}
% \todo[inline]{where is $\Psi$?}
\begin{theorem}\label{le:convergence_minmax}
Under the assumptions \ref{ass:3}, the update equation \ref{eq:update_equation} converges in probability to a unique solution. That is, $\lim_{t \rightarrow \infty} d_t = G^{-1}  h$ in probability. 
% \begin{align*}
%     \lim_{t \rightarrow \infty} d_t = G^{-1}  h
% \end{align*}
\end{theorem}
% \todo[inline]{1. need theorem for "what is the solution of this optimisation problem?" 2. what is the unique solution referred to here?}
The detailed proof is provided in Appendix \ref{proof:convergence_minmax}. The proof is similar to ~\cite[Theorem 2]{DBLP:journals/corr/abs-2001-11113} and invokes theorem 2.2 \citet{10.1137/S0363012997331639}.

We provide a sample complexity analysis for a projected version of the update rule  \eqref{eq:update_equation}. To that end, we propose Algorithm \ref{alg:dice_ldg} called the Projected Log Density Gradient. %~\cite{DBLP:journals/corr/abs-2001-11113}. 
% In this algorithm we assume that our learning variables $(\alpha, \beta, \tau)$ 
We choose closed, bounded and convex sets $X \subset \mathbb{R}^{d \times n}, Y \subset \mathbb{R}^{d \times n}, Z \subset  \mathbb{R}^{1 \times n}$ and define a projection operator $\Pi_X, \Pi_Y, \Pi_Z$ that project our variables $\alpha_{t}, \beta_{t}, \tau_{t}$ onto $X,Y,Z$ respectively. Moreover, we choose a learning rate $\{\varepsilon_t\}_{t=1}^{m}$ where we run the algorithm for $m$ steps. The details of the choice of learning rate are found in Appendix \ref{sec:additional_theory}. 
% To estimate the sample complexity, we must make the following assumption. 
% \todo[inline]{need to discuss and rewrite}
% \begin{assumption}\label{ass:4}
%     The sets $X, Y, Z$ are closed, convex, and bounded sets, and the optimal solution $\alpha_* \in X, \beta^* \in Y, \tau^* \in Z$. 
% \end{assumption}
% \todo[inline]{1. where did $XYZ$ come from? are they user chosen? why are they needed? 2. why are they needed for the sample complexity? 3. simulations to illustrate complexity result is true }
\begin{theorem}\label{th:sample_complexity}
Under assumptions \ref{ass:3} for $(\bar{\alpha},\bar{\beta}, \bar{\tau})$ obtained from Algorithm \ref{alg:dice_ldg} after $m$ steps, the optimality gap $\epsilon_g(\bar{\alpha} \bar{\beta}, \bar{\tau})$ (defined below) is bounded with probability $1 - \delta$ as follows, 
\begin{align*}
    \epsilon_g(\bar{\alpha}, \bar{\beta}, \bar{\tau}) :=  \max_{(\beta, \tau) \in {Y \times Z}} L(\bar{\alpha}, \beta, \tau) - \min_{\alpha \in X} L(\alpha, \bar{\beta}, \bar{\tau}) \leq C_0\sqrt{\frac{5}{m}}(8 + 2 \log \frac{2}{\delta}) \text{ w.p. }  1-\delta
\end{align*}
where, $C_0$ is a constant which is a function of the sets $X, Y, Z$, and the second moment of $\Phi$. 
\end{theorem}
We present the proof of this result in appendix \ref{sec:additional_theory}. This result essentially shows us that the upper-bound for log density gradient estimation requires $ O(\frac{1}{\sqrt{m}})$ (where $m$ is the number of steps the algorithm runs for) samples to learn an accurate estimation. 
\begin{algorithm}[!htb]
  \caption{Projected Log Density Gradient}\label{alg:dice_ldg}   
    \begin{algorithmic}[1]
    \STATE \textbf{for} {$t = 1, 2, ..., m$ do:}
    \STATE $\delta_t = \Phi_t \Phi_t^T - \gamma \Phi_t (\Phi_t')^T$
    \STATE $\alpha_{t+1}^T = \Pi_{X}(\alpha_t^T - \varepsilon_t (\beta^T \delta_t + \lambda (\tau \Phi_t^T ) ) )$
    \STATE $\beta_{t+1}^T = \Pi_{Y}(\beta_t^T + \varepsilon_t( \alpha_t^T \delta_t - g_t\Phi_t^T- \beta_t^T \Phi_t \Phi_t^T))$ 
    \STATE $\tau_{t+1} = \Pi_{Z}(\tau_t + \varepsilon_t(\lambda (\alpha_t^T \Phi_t  - \tau_t)))$
    \STATE \textbf{Return} $\bar{\alpha}$, $\bar{\beta}$, $\bar{\tau}$\\
    Where, $\bar{\alpha} = \frac{\sum_{i=1}^n \varepsilon_i \alpha_i}{\sum_{i=0}^{n}\varepsilon_i}$, $\bar{\beta} = \frac{\sum_{i=1}^n \varepsilon_i \beta_i}{\sum_{i=0}^{n}\varepsilon_i}$, $\bar{\tau} = \frac{\sum_{i=1}^n \varepsilon_i \tau_i}{\sum_{i=0}^{n}\varepsilon_i}$
  \end{algorithmic}
\end{algorithm}
\section{Experiments}\label{sec:experiments}
In this section, we present a proof of concept for our log density gradient estimation on two sets of environments $5 \times 5$ and $3 \times 3$ gridworld environment~\citep{gridworld}. For the gridworld experiments, we approximate log density gradient by using linear function approximation. 
Here, the features are $\phi: \mathcal{S} \times \A \rightarrow \mathbb{R^{|\mathcal{S}| \cdot |\mathcal{A}|}}$ such that it maps every state to the corresponding standard basis vector. 
Our results for $5 \times 5$ are in Figure \ref{fig:gridworld} and for $3 \times 3$ in Appendix \ref{app:num-fig} Figure \ref{fig:gridworld3}.

We compare our algorithm against 3 different baselines. The first is theoretical log density gradient as described in Lemma \ref{le:uniquenss}. The second baseline implements REINFORCE algorithm, which is the practical rendition of the policy gradient theorem ~\citep{DBLP:journals/ml/Williams92}. The third is theoretical policy gradient method which
%exactly computes the classical policy gradient theorem as shown in \eqref{eq:practical_pg2} ~\citep{DBLP:conf/nips/SuttonMSM99}. 
exactly computes the  classical policy gradient theorem, as in \eqref{eq:practical_pg2} ~\citep{DBLP:conf/nips/SuttonMSM99}.

% We observe in all of these cases the log density gradient is more sample efficient than theoretical policy gradient theorem. 
We observe in that both log density gradient approaches are more sample efficient than both policy gradient approaches. 
%
% This is because theoretical policy gradient method approximates the gradient for average reward scenarios ($\gamma=1$) by estimating a Q-function for a discounting factor less than 1. 
This is because policy gradient methods approximate the gradient for average reward scenarios ($\gamma=1$) by estimating a Q-function for a discounting factor less than 1. 
%
% Similarly, we observe that our method  tends to outperform REINFORCE with much reduced variance. In gridworld $5 \times 5$ we also observe LDPG algorithm to even outperform theoretical log density gradient. 
Moreover, we observe that our method tends to outperform REINFORCE with much reduced variance. 
%
% In gridworld $5 \times 5$ we also observe LDPG algorithm to even outperform theoretical log density gradient. This is because, theoretical log density gradient suffers from numerical issues arising from average reward scenarios. 
Our approach is always very close in performance to the theoretical log density gradient which serves to validate correctness of our algorithm.
In $5 \times 5$ gridworld we also observe our algorithm to outperforms theoretical log density gradient. This is because, theoretical log density gradient suffers from some numerical computation issues arising from average reward scenarios. 
% 
% 
% 5 times 5 figure only
 \begin{figure}[!htb]
% \begin{subfigure}{\textwidth}
% \centering
% \includegraphics[width=0.9\textwidth]{images/performance_3_times_3.pdf}
% \caption{}
% \label{fig:ldg_vs_vanilla_theory}
% \end{subfigure}
%\begin{subfigure}{\textwidth}
\centering
\includegraphics[width=\textwidth]{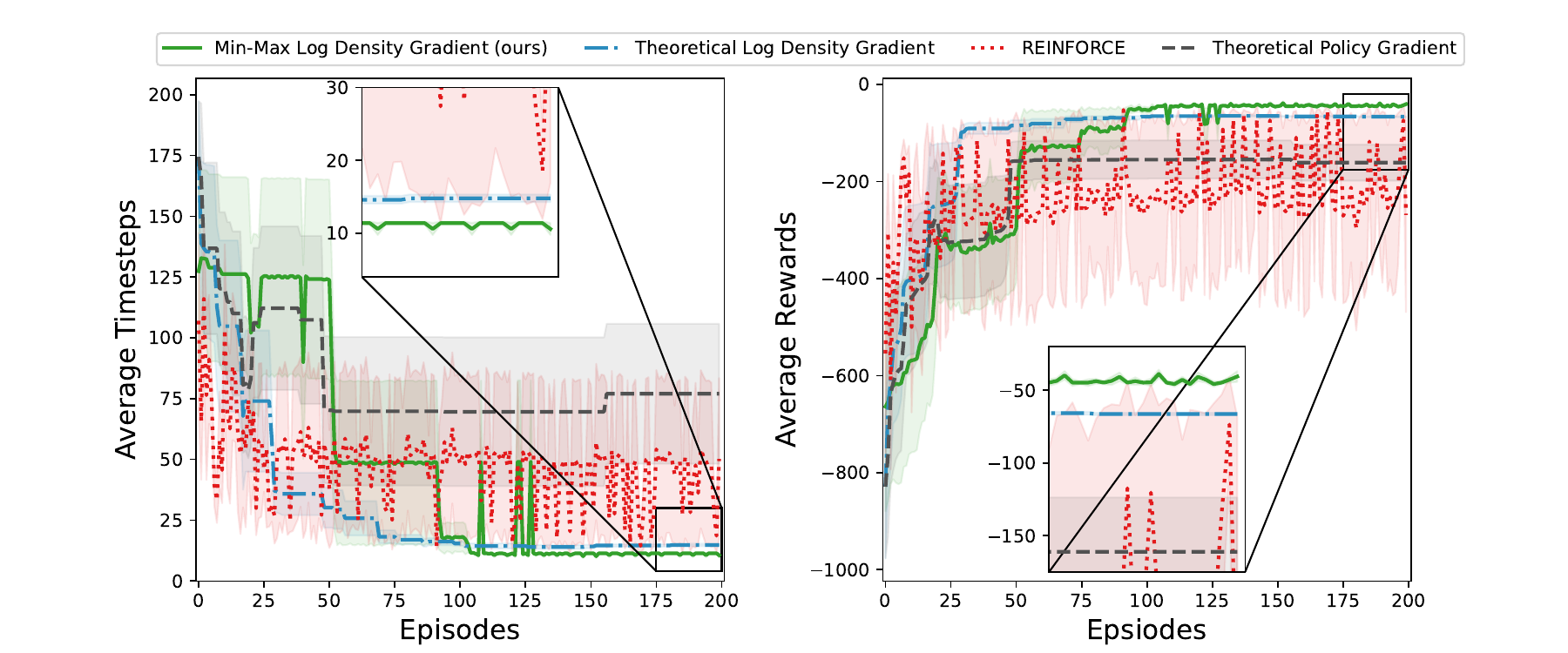}
%\caption{}
%\label{fig:ldg_gridworld}
%\end{subfigure}
\caption{For $5 \times 5$ gridworld, comparison of Log Density Gradient algorithms (in light green) as compared to REINFORCE (light red), theoretical policy gradient (gray) and theoretical log density gradient (blue). 
% We observe that our empirical algorithm comfortably outperforms all the other baselines including even theoretical log density gradient, which suffers from numerical issues.
We observe that our empirical algorithm comfortably outperforms the other baselines.
}
\label{fig:gridworld}
\end{figure}
\begin{figure}[!htb]
% \begin{subfigure}{\textwidth}
% \centering
% \includegraphics[width=0.9\textwidth]{images/performance_3_times_3.pdf}
% \caption{}
% \label{fig:ldg_vs_vanilla_theory}
% \end{subfigure}
%\begin{subfigure}{\textwidth}
\centering
\includegraphics[width=\textwidth]{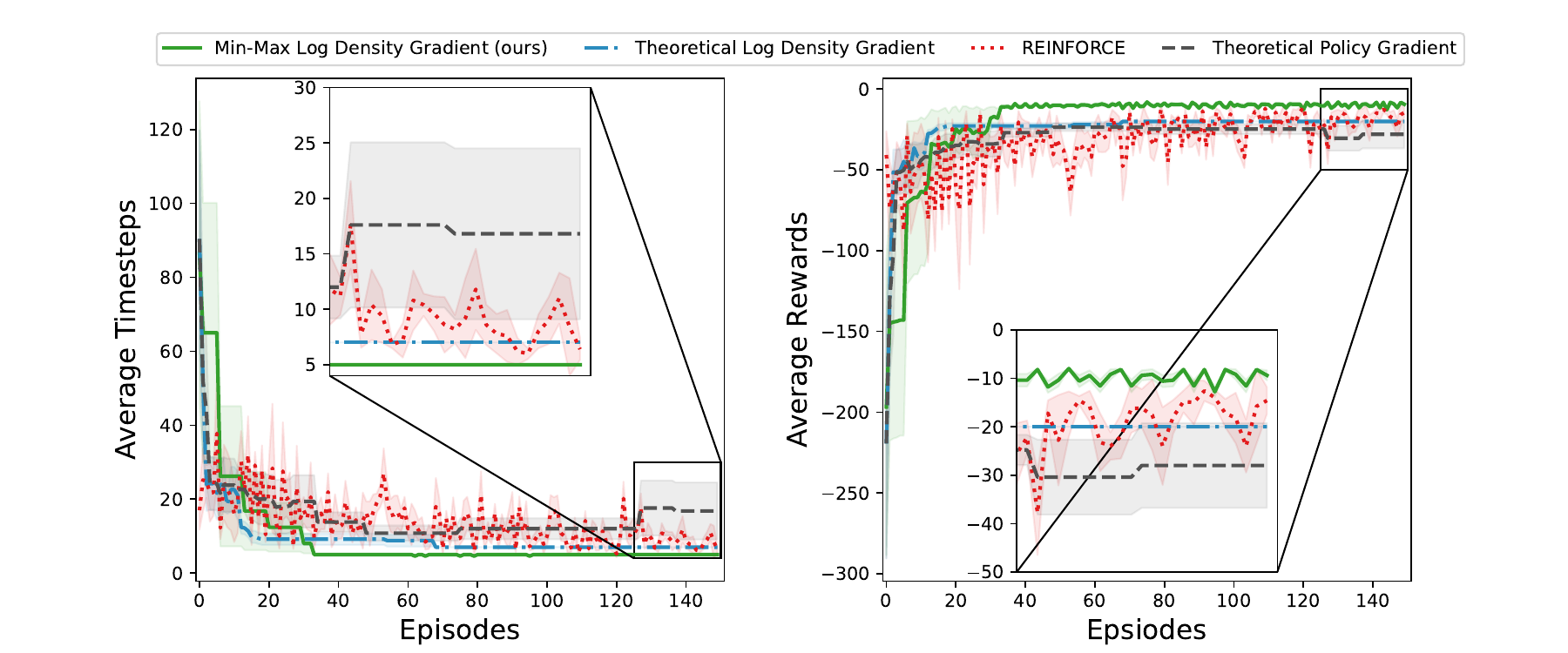}
%\caption{}
%\label{fig:ldg_gridworld}
%\end{subfigure}
\caption{For $3 \times 3$ gridworld, comparison of Log Density Gradient algorithms (in light green) as compared to REINFORCE (light red), theoretical policy gradient (gray) and theoretical log density gradient (blue). 
% We observe that our empirical algorithm comfortably outperforms all the other baselines including even theoretical log density gradient, which suffers from numerical issues.
We observe that our empirical algorithm comfortably outperforms the other baselines.
}
\label{fig:gridworld3}
\end{figure}
\section{Conclusion and Future Work}
We present log density gradient algorithm that estimates policy gradient using state-action discounted formulation of a reinforcement learning problem. 
We observe that policy gradient estimated in this manner, corrects for a residual error common in many reinforcement learning tasks. 
% We show that with a known model, we can exactly calculate the gradient of the log density by solving two sets of linear equations. 
We show that with a known model, we can exactly calculate the gradient of the log density by solving two sets of linear equations. We further propose a TD(0) algorithm to implement the same, but it needs samples from the backward Markov chain, which becomes too restrictive.
% We further propose a min-max optimization that estimates log density gradient using just on-policy samples. 
Therefore, we propose a min-max optimization that estimates log density gradient using just on-policy samples. 
We not only prove theoretical properties like convergence and uniqueness but also experimentally demonstrate that our method is sample efficient as compared to classical policy gradient methods like REINFORCE. 
% Although promising, further studies of log density gradient will focus on scaling their performance to complex tasks.
This approach looks promising, and further studies of log density gradient will focus on scaling their performance to complex tasks.

\pagebreak 
\bibliography{uai2023template}
\bibliographystyle{neurips23}
\pagebreak 
\onecolumn
\section{Appendix}
\subsection{Additional Literature Review - Temporal Difference Methods}\label{sec:td_lit}
Temporal difference (TD) learning involves improving predictions through bootstrapping the current estimation. Early TD methods were mainly used to estimate the value function $V^{\pi_\theta}$, which used a semi-gradient update rule to improve value function prediction (Sutton, 1988). 
%This update rule is known as TD(0) and is expressed as follows, \cite{DBLP:journals/ml/Sutton88}. 
More advanced TD methods incorporate prior traces (TD($\lambda$)) in a discounted manner \citep{DBLP:journals/ml/Tesauro92}.  Since many reinforcement learning methods are not tabular, Linear TD methods were also proposed that uses this methodology to learn parameters that approximate the value function ~\citep{DBLP:journals/ml/BradtkeB96, boyan1999least}. Recently, TD methods have been effectively used to approximate Advantage function \citep{DBLP:conf/icml/MnihBMGLHSK16}, density ratio function \citep{DBLP:conf/aaai/GeladaB19, DBLP:conf/icml/HallakM17, DBLP:conf/aaai/GeladaB19} and off-policy value estimation \citep{gradientTD}.In this work, we will observe that the log density gradient has a recursive form that cannot be approximated using a closed form solution. We will then propose a TD(0) methods to estimate the log density gradient using on-policy samples.
\subsection{Proof of Proposition \ref{prop:consistency}}\label{sec:pg_consistency}
\begin{proof} 
We begin with policy gradient calculated using log density gradient, \eqref{eq:secondform})
\begin{align*}
    \nabla_\theta J_\gamma(\pi_\theta) = \mathbb{E}_{(s, a) \sim d_\gamma^{\pi_\theta}}[\nabla_\theta \log d_\gamma^{\pi_\theta}(s, a)r(s, a)]
\end{align*}
We recall from Bellman equation, \ref{eq:bellman} that $r(s, a) = Q^{\pi_\theta}_\gamma - \gamma \mathbb{E}_{s' \sim \P(\cdot | s, a), a' \sim \pi_\theta(\cdot | s')}[Q^{\pi_\theta}_\gamma(s', a')]$, hence
\begin{align*}
    \nabla_\theta J_\gamma(\pi_\theta) = \EE_{(s, a) \sim d_\gamma^{\pi_\theta}}[\nabla_\theta \log d_\gamma^{\pi_\theta}(s, a)(Q_\gamma^{\pi_\theta}(s, a) - \gamma \mathbb{E}_{s' \sim \mathcal{P}(\cdot | s, a) \pi_\theta(a' | s')}[Q_\gamma^{\pi_\theta}(s', a')])
\end{align*}
We recover the policy gradient theorem in \eqref{eq:policy_gradient} by first multiplying \eqref{eq:ldg_condition} by $Q^{\pi_\theta}_\gamma(s', a')$
\begin{align*}
    & d_\gamma^{\pi_\theta}(s', a') (\nabla_\theta \log d_\gamma^{\pi_\theta}(s', a') - \nabla_\theta \log \pi(a' | s'))Q^{\pi_\theta}_\gamma(s', a')  \\
    &= \gamma \sum_{s, a}d_\gamma^{\pi_\theta}(s, a) \nabla_\theta \log d_\gamma^{\pi_\theta}(s, a) \mathcal{P}(s' | s, a) \pi_\theta(a' | s') Q^{\pi_\theta}_\gamma(s', a')
\end{align*}
and then re-arranging terms with $\nabla_\theta \log d^{\pi}_\gamma$ on the left hand side and the remaining terms on the right hand side. We can then sum these terms to get \eqref{eq:policy_gradient}
\begin{align*}
    \nabla_\theta J_\gamma (\pi_\theta) &= \EE_{(s, a) \sim d_\gamma^{\pi_\theta}}[\nabla_\theta \log d_\gamma^{\pi_\theta}(s, a)(Q_\gamma^{\pi_\theta}(s, a) - \gamma \mathbb{E}_{s' \sim \mathcal{P}(\cdot | s, a) \pi_\theta(a' | s')}[Q_\gamma^{\pi_\theta}(s', a')]) \\ 
&= \mathbb{E}_{(s, a) \sim d^{\pi_\theta}_\gamma}[Q^{\pi_\theta}_\gamma(s, a) \nabla_\theta \log \pi ( a | s)]
\end{align*}
This completes the proof. 
%     \begin{align*}
% &\mathbb{E}_{(s, a) \sim d_\gamma^{\pi_\theta}}[\nabla_\theta \log \pi_\theta (a | s) Q_\gamma^{\pi_\theta}(s, a)]
%     \\
%     &=  \EE_{(s, a) \sim d_\gamma^{\pi_\theta}}[\nabla_\theta \log d_\gamma^{\pi_\theta}(s, a)(Q_\gamma^{\pi_\theta}(s, a) - \gamma \mathbb{E}_{s' \sim \mathcal{P}(\cdot | s, a) \pi_\theta(a' | s')}[Q_\gamma^{\pi_\theta}(s', a')]) \\
%     &=\mathbb{E}_{(s, a) \sim d_\gamma^{\pi_\theta}}[\nabla_\theta \log d_\gamma^{\pi_\theta}(s, a)r(s, a)] \\
%     &= \nabla_\theta J_\gamma(\pi_\theta)
% \end{align*}
\end{proof}
\subsection{Proof of Lemma \ref{le:bellman_conservation}}\label{proof:bellman_conservation}
\begin{proof}
Recall the definition of $d_\gamma^{\pi_\theta}$ from  \eqref{eq:discounted_form} 
\begin{align*}
    &d_\gamma^{\pi_\theta}(s) = (1-\gamma)\sum_{t=0}^{\infty} \gamma^t \Prob(s_t = s | s_0 \sim d_0, a_t \sim \pi_\theta(s_t), s_{t+1} \sim \mathcal{P}(s_t, a_t))
\end{align*}
We see that for a given state $s' \in \S$ the occupancy of an agent at time $t$ and $t+1$ are related as follows, 
\begin{align*}
    \Prob(s_{t+1} = s') = \sum_{s \in \S, a \in \A} \Prob(s_t = s)\cdot \pi_\theta(a | s) \cdot \P(s' | s, a)
\end{align*}
If we multiply both sides by $\gamma^{t+1}$ and sum them up from $t=0$ to $\infty$. We get the following form, 
\begin{align*}
    \sum_{t=0}^{\infty} \gamma^{t+1} \Prob(s_{t+1} = s')  &= \gamma \sum_{t=0}^{\infty} \gamma^t \sum_{s \in \S, a \in \A} \Prob(s_t = s)\pi_\theta(a | s) \P(s' | s, a) \\ 
     &= \gamma \sum_{s \in \S, a \in \A} \sum_{t=0}^{\infty} \gamma^t \Prob(s_t = s)\pi_\theta(a | s) \P(s' | s, a) \\
     &  = \frac{\gamma}{1-\gamma} \sum_{s \in \S, a \in \A} d_\gamma^{\pi_\theta}(s) \pi_\theta(a | s) \P(s' | s, a) \\
\end{align*}
We go from the first equation to the second by exchanging the summation signs and from the second to the third  by using the definition of $\sum_{t=0}^{\infty} \gamma^t \Prob(s_t = s) = d_\gamma^{\pi_\theta}$. 
%Since, the summation on the LHS goes from $t = 1$ to $\infty$, 
We add $\Prob(s_0 = s')$ on both sides, which is nothing but the set of initial states $d_0(s')$, to get
\begin{align*}
    &\sum_{t=0}^{\infty} \gamma^{t} \Prob(s_{t} = s')  =  d_0(s') + \frac{\gamma}{1-\gamma} \sum_{s \in \S, a \in \A} d_\gamma^{\pi_\theta}(s) \pi_\theta(a | s) \P(s' | s, a) \\
    & \implies \frac{1}{1-\gamma}d_\gamma^{\pi_\theta}(s')  =  d_0(s') + \frac{\gamma}{1-\gamma} \sum_{s \in \S, a \in \A} d_\gamma^{\pi_\theta}(s) \pi_\theta(a | s) \P(s' | s, a) \\
    & \implies d_\gamma^{\pi_\theta}(s')  =  (1-\gamma)d_0(s') + \gamma\sum_{s \in \S, a \in \A} d_\gamma^{\pi_\theta}(s) \pi_\theta(a | s) \P(s' | s, a) \\
\end{align*}
where we used 
\begin{align*}
    \sum_{t=0}^{\infty} \gamma^{t+1} \Prob(s_{t+1} = s') + d_0(s') = \sum_{t=0}^{\infty} \gamma^{t} \Prob(s_{t} = s') = \frac{1}{1-\gamma}d_\gamma^{\pi_\theta}(s').
\end{align*}
This completes the proof. 
\end{proof}
\subsection{Proof of Lemma \ref{le:bellman_conservation_uniqueness}}\label{proof:bellman_conservation_uniqueness}
\begin{proof}
Recall the optimization \ref{eq:bellman_unique_1} 
\begin{align*}
    \argmin_{w: \S \rightarrow \mathbb{R}}\sum_{s'} \Large( w(s') - (1-\gamma)d_0(s') + \gamma \sum_{s, a} w(s) \pi_\theta(a | s) \P(s' | s, a)\Large)^2 + \frac{\lambda}{2} (\sum_{s} w(s) - 1)^2 
\end{align*}
It is worth noting that the optimization has two quadratic terms. Thus, the lowest value that they can take is only when $w(s') - (1-\gamma)d_0(s') + \gamma \sum_{s, a} w(s) \pi_\theta(a | s) \P(s' | s, a) = 0$ and $\sum_{s} w(s) = 1$. For the next part of the proof, we will handle them case-by-case. 

\textbf{Case 1. $\gamma < 1$} 

Recall that the first term of the optimization \eqref{eq:bellman_unique_1} goes to zero only when 
\eqref{eq:gradient_identity} (reproduced below) is satisfied 
\begin{align*}
    w(s') = (1-\gamma)d_0(s') + \gamma \sum_{s \in \S, a \in \A} w(s)\pi_\theta(a | s)\P(s' | s, a)  \quad \text{ for all } s' \in \S.
\end{align*}
Now all that remains to be proven is the uniqueness aspect of this form. To that end, we first write the above set of equations in a matrix form 
% for $\forall \S$. 
\begin{align}
    W = (1-\gamma)D_0 + \gamma  \P_{\pi_\theta}^T W,
\end{align}
where $W = (w(s_1), w(s_2) ,....., w({s_{|\S|}})) \in \R^{|\S|}$, 
$\P_{\pi_\theta}, \D \in \mathbb{R}^{|\S|  \times |\S| }$ where $ (\P_{\pi_\theta})_{(s,s')} = \Prob^{\pi_{\theta}}(s'|s)$ and $\D$ is a diagonal matrix whose every element correspond to $d_0$.
% We similarly encode the set of initial distribution as $D_0$ and the transition probability in form of a matrix $\P_{\pi_\theta}$ such that $(\P_{\pi_\theta})_{ij} = \P(j | i)$. Note that $\P_{\pi_\theta}$ is the transition matrix of Markov chain induced by $\pi_\theta$. 
% where, $W \in \mathbb{R}^{|\S| \cdot | \A | \times n}$ is the matrix with every row corresponding to $w(s,a)$ for each state-action pair $(s,a)$. Similarly, $G \in \mathbb{R}^{|\S| \cdot | \A | \times n}$ has its rows as $\nabla_\theta \log \pi_\theta$ for each state-action pair. Let $\P_{\pi_\theta}, \D \in \mathbb{R}^{|\S| \cdot | \A | \times |\S| \cdot | \A |}$ where $ (\P_{\pi_\theta})_{((s,a),(s',a'))} = \Prob^{\pi_{\theta}}(s',a'|s,a)$ and $\D$ is a diagonal matrix whose every element correspond to $d_\gamma^{\pi_\theta}$ for each  state-action pair. We use this matrix form for the operator $Y_{\gamma}$ in the proof of the following lemma.
It is easy to see from this form here that the row sum of this matrix is 1. We can simply re-write the above equation as follows, 
\begin{align*}
    (I_{| \S |} - \gamma \P_{\pi_\theta}^T)W = (1-\gamma)D_0
\end{align*}
$I_{|\S|}$ is the identity matrix of size $|\S|$. It now remains to prove that $(I_{|\S| } - \gamma \P_{\pi_\theta}^T) $ is invertible. We propose the following Lemma which proves the same result. 
\begin{lemma}\label{le:p_unique}
    For $\gamma < 1$, The matrix $(I_{|\S| } - \gamma \P_{\pi_\theta}^T)$ is invertible. 
\end{lemma}
\begin{proof}
It is generally easier to prove that the transpose of this matrix is invertible. Consider $x \in \mathbb{R}^{|\S|}$ is a non-zero vector. We will now prove that $(I_{|\S| } - \gamma \P_{\pi_\theta}) x $ is non-zero. To see that observe the infinty norm as follows, 
\begin{align*}
    \|(I_{|\S| } - \gamma \P_{\pi_\theta}) x\|_\infty &\geq \| x\|_\infty - \gamma \|\P^{\pi_\theta}x\|_\infty\\
    & \geq (1-\gamma) \| x\|_\infty > 0
\end{align*}
We get the first equation from the triangular inequality. We get the second equation from the fact that the row sum of the matrix $\P_{\pi_\theta}$ is 1. Thus, $\|\P_{\pi_\theta} \cdot x\|_\infty \leq \|x\|_\infty$. This implies that $(I_{|\S|} - \gamma \P_{\pi_\theta}^T)$ is invertible. This completes the proof. 
\end{proof}

Since the matrix $(I_{|\S|} - \gamma \P_{\pi_\theta}^T)$ is invertible 
%and the solution to the \eqref{eq:gradient_identity} is unique 
the solution is equal to $d^{\pi_\theta}_\gamma$. The solution also satisfies the constraint $\sum_s w(s) = \sum_s d^{\pi_\theta}(s) = 1$. This completes the proof for the first case. 

\textbf{Case 2. $\gamma = 1$} \\
For $\gamma = 1$ we need to solve 
\begin{align*}
   (I - \P_{\pi_\theta}^T) W = 0
\end{align*}
The proof now follows from ~\cite[Theorem 1]{DBLP:journals/corr/abs-2001-11113}. Here we invoke Perron-Frobenius theorem which says the dimension of the left eigenspace of$\P_{\pi_\theta}$ corresponding to eigevalue 1 is one-dimensional. Since $d^{\pi_\theta}_1$ belongs in that set,the solution of these set of equations as $W = \alpha d^{\pi_\theta}_1$. 
% 
% Here, since $W$ is the right Eigen vector of $\P_{\pi_\theta}$ and that $d^{\pi_\theta}_1$ is another eigen-vector of the solution. We thus use Perron-Frobenius theorem which says the space associates with left eigen space is one-dimensional. Thus, we have the solution of these set of equations as $W = \alpha d^{\pi_\theta}_1$.
% \begin{align*}
%     W = \alpha d^{\pi_\theta}_1
% \end{align*}
It is worth noting that only $\alpha = 1$ satisfies the second constraint which is $\sum_{s} w(s) = 1$. This completes the proof. 
\end{proof}
\subsection{Proof of Lemma \ref{le:uniquenss}}\label{proof:uniquenss}
% \begin{lemma}
% Given that state-action spaces are finite $\S, \A$ and that, the policy $\pi_\theta$ induces a distribution $d^{\pi_\theta}_\gamma(s, a) > 0, \forall (s, a) \in \S \times \A$. The solution to the following set of equations is unique and equal to $\nabla_\theta \log d_\gamma^{\pi_\theta}$ for $\forall \gamma \in [0, 1]$.
% \begin{align*}
%     w^* &=  \min_{w: \mathcal{S} \times \mathcal{A} \rightarrow \mathbb{R}^{n}} \EE_{(s', a') \sim d_\gamma^{\pi_\theta}} \|d_\gamma^{\pi_\theta}(s', a') (w(s', a') 
%     - \nabla_\theta \log \pi_\theta(a' |s'))\\
%     &- \gamma \sum_{s, a}d_\gamma^{\pi_\theta}(s, a) \mathcal{P}(s' | s, a) \pi_\theta(a'|s') w(s, a) \| ^2  + \frac{\lambda}{2} \| E_{(s, a) \sim d^{\pi_\theta}_\gamma}[w(s', a')] \|^2
% \end{align*}
% \end{lemma}
This proof is very similar to proof of Lemma \ref{le:bellman_conservation_uniqueness}. 
\begin{proof}
Let us restate the optimization problem in Lemma \ref{le:uniquenss} (\eqref{eq: opt_ldg}):
\begin{align*}
    &\min_{w: \mathcal{S} \times \mathcal{A} \rightarrow \mathbb{R}^{n}} 
\bigg\{ \EE_{(s', a') \sim d_\gamma^{\pi_\theta}} \| \nu(s',a') \| ^2   + \frac{\lambda}{2} \Big\| E_{(s, a) \sim d^{\pi_\theta}_\gamma}[w(s', a')] \Big\|^2 \bigg\} \\
    &\nu(s',a') := d_\gamma^{\pi_\theta}(s', a') (w(s', a') - \nabla_\theta \log \pi_\theta(a' |s')) - \gamma \sum_{s, a}d_\gamma^{\pi_\theta}(s, a) \mathcal{P}(s' | s, a) \pi_\theta(a'|s') w(s, a) 
\end{align*}
We note that it has two quadratic terms.
As in Lemma \ref{le:bellman_conservation_uniqueness}, this loss function can only go to zero if and only if both the quadratic terms turn out to be zero. This implies that 
\begin{align*}
    &\nu(s', a') = 0 \, \text{ for all } (s', a') \in \S \times \A \quad \text{ and } \quad \mathbb{E}_{(s', a') \sim d^{\pi_\theta}_\gamma}[w(s', a')] = 0.
\end{align*}
We take two cases. 

\textbf{Case 1. $\gamma < 1$} 

Similar to Lemma \ref{le:bellman_conservation_uniqueness}, for a finite state-action space we can re-write equation \ref{eq: opt_ldg} in a linear form
\begin{align}\label{eq:ldg_matrix_form}
    (I_{|\S| \times | \A | } - \gamma \P_{\pi_\theta}^T) \, D_{\pi_\theta} \, W = D_{\pi_\theta} \, G 
\end{align}
% where, $W \in \mathbb{R}^{|\S| \cdot | \A | \times n}$ is the matrix with every row corresponding to $w(s,a)$ for each state-action pair $(s,a)$. Similarly, $G \in \mathbb{R}^{|\S| \cdot | \A | \times n}$ has its rows as $\nabla_\theta \log \pi_\theta$ for each state-action pair. Let $\P_{\pi_\theta}, \D \in \mathbb{R}^{|\S| \cdot | \A | \times |\S| \cdot | \A |}$ where $ (\P_{\pi_\theta})_{((s,a),(s',a'))} = \Prob^{\pi_{\theta}}(s',a'|s,a)$ and $\D$ is a diagonal matrix whose every element correspond to $d_\gamma^{\pi_\theta}$ for each  state-action pair. We use this matrix form for the operator $Y_{\gamma}$ in the proof of the following lemma.
% where, $W = \begin{bmatrix}
%     w_1^T \\
%     w_2^T \\
%     ... \\
%     w_n^T
% \end{bmatrix}$ and $G = \begin{bmatrix}
%     g_1^T \\
%     g_2^T \\ 
%     ... \\ 
%     g_n^T
% \end{bmatrix}$. Note that rows in $G$ corresponds to $\nabla_\theta \log {\pi_\theta}(a | s)$ for a fixed $(s, a)$. The matrix $\P_{\pi_\theta} \in \mathbb{R}^{|\S| \cdot |\A | \times |\S| \cdot | \A |}$ is the same transition matrix that we described in section \ref{proof:bellman_conservation_uniqueness}. 
% such that $(\P_{\pi_\theta})_{ij} = \mathcal{P}(i | j)$, $D_{\pi_\theta}$ is diagonal matrix such that every element in the diagonal corresponds to $d^{\pi_\theta}_\gamma$. 
where, $W \in \mathbb{R}^{|\S| \cdot | \A | \times n}$ is the matrix with every row corresponding to $w(s,a)$ for each state-action pair $(s,a)$. Similarly, $G \in \mathbb{R}^{|\S| \cdot | \A | \times n}$ has its rows as $\nabla_\theta \log \pi_\theta$ for each state-action pair. Let $\P_{\pi_\theta}, \D \in \mathbb{R}^{|\S| \cdot | \A | \times |\S| \cdot | \A |}$ where $ (\P_{\pi_\theta})_{((s,a),(s',a'))} = \Prob^{\pi_{\theta}}(s',a'|s,a)$ and $\D$ is a diagonal matrix whose every element correspond to $d_\gamma^{\pi_\theta}$ for each  state-action pair.
The ergodicity assumption implies that $D_{\pi_\theta}$ is invertible. 
Additionally, we also proved in Lemma \ref{le:p_unique} in Section \ref{proof:bellman_conservation_uniqueness} that $(I - \gamma \P_{\pi_\theta}^T)$ is invertible. 
This ensures that the solution of equation \ref{eq:ldg_matrix_form} is unique and equal to the log density gradient. Since, log density gradient always satisfies the constraint $\mathbb{E}_{(s, a) \sim d^{\pi_\theta}_\gamma}[\nabla_\theta \log d^{\pi_\theta}(s, a)] = 0$, the second constraint becomes redundant. This completes the proof.

\textbf{Case 2. $\gamma = 1$}

For $\gamma = 1$ we have the following set of equations 
\begin{align*}
    (I_{|\S| \times | \A | } - \P_{\pi_\theta}^T) \, D_{\pi_\theta} \, W = D_{\pi_\theta} \, G  \quad \text{ and } \quad e^T D_{\pi_\theta} W = 0.
\end{align*}
Let us look at the equations column by column.
The matrix $ (I_{|\S| \times | \A | } - \P_{\pi_\theta}^T) $ is not invertible since the vector $d_{1}^{\pi_{\theta}}$ is in its nullspace. 
% We also know that any solution for a linear equation $Ax = b$ can be written as a sum of a particular solution and any element in the null space of matrix $A$. 
Since $\nabla_\theta \log d^{\pi_\theta}_1$ satisfies the first equation, every column can be written as $\nabla_\theta \log d^{\pi_\theta}_\gamma + v $ where $v$ is any vector in the span of $\{d_{1}^{\pi_{\theta}}\}$.  
% Thus, we can write the solution of above set of equations as a particular solution (which can be $\nabla_\theta \log d^{\pi_\theta}_1$ as any element in the null space
% \begin{align*}
%     \{\nabla_\theta \log d^{\pi_\theta}_\gamma + C \quad | \quad (I_{|\S| \times | \A | } - \P_{\pi_\theta}^T) \cdot D_{\pi_\theta} \cdot C = 0\} 
% \end{align*}
This implies that $D_{\pi_\theta} v = \P_{\pi_\theta}^T D_{\pi_\theta} v$. invoking Perron-Frobenius theorem as in Lemma \ref{proof:bellman_conservation_uniqueness}, we see that $v$ lies in the span of $\{d_1^{\pi_{\theta}}\}$. 
% $v = \alpha D_{\pi_\theta}^{-1} d^{\pi_\theta}_1 = \alpha$. 
We complete the proof by trying to satisfy the second constraint $e^T D_{\pi_\theta} v = 0$ which gives us $v = 0$. This completes the proof. 
\end{proof}

\subsection{Proof of Lemma \ref{le:fixed_point}}\label{sec:fixed_point}

% \todo[inline]{1. unclear what is $\mathcal{P}_{\pi_{\theta}}$ and what is its size 2. what is the size of $D$? \\
% \textbf{Answer:} Basically if you are able to index each and every state-action pair from 0 to n, wherein $\forall (s, a), \exists i \in \{0, ...n \}$. You can now define the transition probability as a matrix which we call $\mathcal{P}_{\pi_\theta}$, $\mathcal{D}_{\pi_\theta}$ given the indexing of the states as described above, this term is a diagonal matrix where every diagonal corresponds to $d^{\pi_\theta}$ }
\begin{proof}
We will prove contraction first for $\gamma \in [0,1)$. 

\textbf{Proof of Contraction:} 
To prove contraction, we will show that given any arbitrary functions $U, V: \S \times \A \rightarrow \mathbb{R}^{n}$, their difference under the L1 norm of the distribution $d^{\pi_\theta}_\gamma$ is a contraction. 
    \begin{align*}   
                 &\sum_{s', a'} d^{\pi_\theta}_\gamma(s' ,a') | Y_{\gamma} \cdot U(s', a') - Y_{\gamma} \cdot V(s', a')|  \\
                 &= \sum_{s', a'} d^{\pi_\theta}_\gamma(s' ,a')|\gamma \sum_{s, a} d^{\pi_\theta}_\gamma(s, a) (U(s, a) - V(s, a) )\mathcal{P}(s' | s, a) \pi_\theta(a' | s')| \\
                 & \leq \gamma \sum_{s', a'} d^{\pi_\theta}_\gamma(s' ,a')\sum_{s, a}d^{\pi_\theta}_\gamma(s, a) | U(s, a) - V(s, a)|\mathcal{P}(s' | s, a)\pi_\theta (a' | s') \\
                 & \leq \gamma \sum_{s, a} d^{\pi_\theta}_\gamma(s, a)\sum_{s', a'} d^{\pi_\theta}_\gamma(s' ,a')\ | U(s, a) - V(s, a) | \mathcal{P}(s' | s, a) \pi_\theta (a' | s') \\
                 & \leq \gamma \sum_{s, a} d^{\pi_\theta}_\gamma(s, a)\sum_{s', a'} \ | U(s, a) - V(s, a) | \mathcal{P}(s' | s, a) \pi_\theta (a' | s') \\
                 &\leq \gamma \sum_{s, a} d^{\pi_\theta}_\gamma(s, a) | U(s, a) - V(s, a)| = \gamma \|U-V\|_{1}^{d^{\pi_\theta}_\gamma}.                 
    \end{align*}
This completes the proof. The contraction property is useful in proving uniqueness and convergence which we prove next. 

\textbf{Case 1. $\gamma < 1$} 

% Although the proof of Lemma also follows directly from contraction property of the operator $Y_\gamma$. We will prove this here in another way. 
Let $Y_\gamma^k$ denote that $Y_\gamma $ has been composed with itself $k$ times. We are interested in what happens when $\lim_{k \rightarrow \infty} Y_\gamma^k$. 
\begin{align*}
    Y_{\gamma}^k C_0 = \gamma^k D_{\pi_\theta}^{-1} (\P_{\pi_\theta}^T)^k D_{\pi_\theta}C_0 + \sum_{t=0}^{k-1} \gamma^{t} D_{\pi_\theta}^{-1} (\P_{\pi_\theta}^T)^t D_{\pi_\theta} G
\end{align*}
Noting that $\P_{\pi_\theta}$ is a transition probability matrix so every element is bounded above by 1, and since $\gamma < 1$, we get that $\lim_{k \rightarrow \infty} \gamma^k D_{\pi_\theta} (\P_{\pi_\theta}^T)^k D_{\pi_\theta}C_0 = 0$. Focusing on the other terms as $k \rightarrow \infty$ we get the following, 
\begin{align*}
    \lim_{k \rightarrow \infty} &\sum_{t=0}^{k-1} \gamma^{t} D_{\pi_\theta}^{-1} (\P_{\pi_\theta}^T)^t D_{\pi_\theta} G = D_{\pi_\theta}^{-1}(I - \gamma \P_{\pi_\theta}^T)^{-1}D_{\pi_\theta} G
\end{align*}
From \eqref{eq:ldg_matrix_form} we see that 
$\lim_{k \to \infty}Y_{\gamma}^k C_0 = \nabla_\theta \log d_\gamma^{\pi_\theta}$ which completes the proof. 
\iffalse
\textbf{Case 2. $\gamma = 1$} 

Without loss of generality we will assume that the parameter for space for $\theta$ is one dimensional, that is, $n=1$. The extension to a general value of $n$ is straight forward. Let us look at $Y_1^k C_0$. We have
\begin{align*}
    Y_1^k C_0 =  D_{\pi_\theta}^{-1} ((\P_{\pi_\theta})^T)^k D_{\pi_\theta}C_0 + \sum_{t=0}^{k-1}  D_{\pi_\theta}^{-1} ((\P_{\pi_\theta})^T)^t D_{\pi_\theta} G.
\end{align*}
Define 
\begin{align*}
    L := e^T D_{\pi_\theta} C_0, \quad V_0 := \frac{D_{\pi_\theta} C_0}{L}.
\end{align*}
From \citep{DBLP:books/sp/MeynT93, DBLP:conf/aaai/GeladaB19} we know that $\lim_{k \rightarrow \infty} (\P_{\pi_\theta}^T)^k V_0 = d^{\pi_\theta}_1$. 
Hence
\begin{align*}
    \lim_{k \rightarrow \infty} D_{\pi_\theta}^{-1} ((\P_{\pi_\theta})^T)^k D_{\pi_\theta}C_0 = Le =: \kappa.
\end{align*}
%
% We now have, $\lim_{k \rightarrow \infty} = D_{\pi_\theta}^{-1}(\P_{\pi_\theta}^T)^k V_0 L = Le$, 
% where $V_0 := \frac{D_{\pi_\theta} C_0}{L}$ and $L := e^T D_{\pi_\theta} C_0$. 
% Note that, we do this to normalize $V$. Thus, $\lim_{k \rightarrow \infty} (\P_{\pi_\theta}^T)^k V = d^{\pi_\theta}_1$~\citep{DBLP:books/sp/MeynT93, DBLP:conf/aaai/GeladaB19}. 
Lets now look at the second term. We see that these set of terms have nothing to do with the initial guess. It also turns out to be the definition of $\nabla_\theta \log d^{\pi_\theta}_1$. Which completes the proof. 
\fi
\end{proof}
\subsection{Additional Details Linear-TD}\label{sec:linear_td_details}
\todo[inline]{how do you know the limit exists? it follows from Lemma above right?}
% \todo[inline]{can you write a small proof for the above equation? \\ 
% \textbf{Answer:} What I mean is this as follows, 
% $Y_1^\infty C_0 = \nabla_\theta \log d^{\pi_\theta} + K$. Since, we know that $\mathbb{E}_{d^{\pi_\theta}}[\nabla_\theta \log d^{\pi_\theta}] = 0$. We can get $\nabla_\theta \log d^{\pi_\theta} = Y^\infty \cdot C_0 - \mathbb{E}_{d^\pi_\theta}[Y^\infty \cdot C_0]$}
%We can further show that the operator $Y_\gamma$ is an $L_1$ contraction under the probability measure  $d_\gamma^{\pi_\theta}$.
% \begin{lemma}\label{le:contraction_property}
% For $\gamma < 1$, the operator ${Y}_\gamma$ is an $L_1$ contraction under the probability measure  $d_\gamma^{\pi_\theta}$.
% \end{lemma}
% It is worth noting that although ${Y}_1^k \cdot w$ converges as $k \to \infty$ to $\nabla_\theta \log d_1^{\pi_\theta}$, the operator itself is \textit{not} a contraction \ref{le:contraction_property}. We present the proof for this Lemma in 
\todo[inline]{combine lemma 4 and 5}
% \iffalse
% \begin{theorem}\label{th:convergence}
% For $\gamma < 1$, the operator $c_{k+1} = \tilde{Y}_\gamma c_{k}$ converges to $\nabla_\theta \log d_\gamma^{\pi_\theta}$ under $L_1$ norm with respect to the distribution $d_\gamma^{\pi_\theta}$
% \end{theorem}
% We present the detailed proof in appendix \ref{proof:convergence} which uses contraction property to ensure convergence. An apparent drawback of using this particular form is that the error in $\nabla_\theta \log d_\gamma^{\pi_\theta}$ can be arbitrary points where $d_\gamma^{\pi_\theta}$ is small. This implies that our method would tolerate more error where $d_\gamma^{\pi_\theta}$ is smaller. The problem particularly exacerbates the reward function $r(s, a)$ is significantly larger at these points. 
% \fi
% \iffalse
% \begin{algorithm}[t]
%   \caption{Model-Based Log Density Policy Gradient}\label{alg:exact_ldg}                          
%   \begin{algorithmic}[1]
%   \REQUIRE{$\pi_\theta$, $\mathcal{P}$, $d_0$}
%     \STATE Estimate $d_\gamma^{\pi_\theta}$ from equation \ref{eq:gradient_identity} 
%     \STATE Estimate $\nabla_\theta \log d_\gamma^{\pi_\theta}$ from equation 
%     \ref{eq:ldg_condition}
%     \STATE Calculate $\nabla_\theta J(\pi_\theta) = \mathbb{E}_{(s, a) \sim d_\gamma^{\pi_\theta}}[\nabla_\theta \log d_\gamma^{\pi_\theta}(s, a) \cdot r(s, a)]$
%     \STATE \textbf{Return} $\nabla_\theta J(\pi_\theta)$
%   \end{algorithmic}
% \end{algorithm}
% \fi
For large problem spaces, it is difficult to learn $w(s', a')$ for all the state-action spaces. In that case, we use linear function approximation, see for example, ~\cite{DBLP:conf/aaai/GeladaB19, DBLP:conf/icml/HallakM17}.
We choose a feature map $\Phi: \mathcal{S} \times \mathcal{A} \to \mathbb{R}^{d}$  and $\zeta \in \mathbb{R}^{d \times n}$ are the linear parameters that we wish to learn. We want to approximate $w(s,a)$ as $\zeta^T \Phi(s, a)$ for each state action pair $(s,a)$. 
\todo[inline]{where is $\Psi$? }
%We can similarly define $C$ to be the function values for all state-action spaces, $C = \zeta^T \Psi$, 
% 
% \iffalse
% \begin{remark}
%     It is important to note that the function $c(\cdot, \cdot)$ should satisfy the following condition. 
%     \begin{align*}
%         \mathbb{E}_{(s, a) \sim d_\gamma^{\pi_\theta}}[c(s, a)] = \mathbb{E}_{(s, a) \sim d_\gamma^{\pi_\theta}}[\zeta^T \Phi(s, a)] = 0
%     \end{align*}
% Since $\zeta$ is a matrix the above condition does not necessarily imply that $\zeta = 0$. In practice, there might be a need to normalize $\zeta$ such that this condition is satisfied. 
% \end{remark}
% \fi
The TD(0) algorithm with linear function approximation has the following update
\begin{align}\label{eq:linear_td}
    \zeta_{k+1}^T \leftarrow \zeta_{k}^T + \alpha_{k}(\gamma \zeta_k^T \Phi(s_k, a_k)  + g(s_k', a_k') - \zeta_k^T \Phi(s_k', a_k'))\Phi(s_k', a_k')^T,
\end{align}
where $(s_k, a_k) \sim d_\gamma^{\pi_\theta}, s_k' \sim \mathcal{P}(\cdot | s_k, a_k), a_k' \sim \pi_\theta(\cdot | s_k')$ and $\alpha_k$ is the learning rate. 
To prove convergence, re-write \eqref{eq:linear_td} in the following linear form,
% \todo[inline]{"such a linear TD approach"? write exact equation number.}
\begin{align*}
    &\zeta_{k+1} = \zeta_k + \alpha_{k}(A_{k+1} \zeta_k + g_{k+1}) 
\end{align*}
where, 
% \begin{align*}
%    A_{k+1} = \gamma \Phi(s', a') \Phi(s, a)^T - \Phi(s', a') \Phi(s', a')^T , g_{k+1} = \Phi(s', a') (\nabla_\theta \log \pi_\theta (a' | s'))^T 
% \end{align*}
$A_{k} := \gamma \Phi(s_k', a_k')(\Phi(s_k, a_k)^T - \Phi(s_k', a_k')^T) , g_{k} := \Phi(s_k', a_k') (\nabla_\theta \log \pi_\theta (a_k' | s_k'))^T$. 
\todo[inline]{new variable for $g$, and write the grad log pi term as g to avoid confusion}
We further define, 
\begin{align*}
    A &:= \gamma \mathbb{E}[\Phi(s', a')( \Phi(s, a)^T - \Phi(s', a')^T) \, | \, (s, a) \sim d_\gamma^{\pi_\theta}, s' \sim \mathcal{P}(\cdot | s, a), a' \sim \pi_\theta(\cdot | s')], \\
    g &:= \mathbb{E}[\Phi(s', a') (\nabla_\theta \log \pi_\theta (a' | s'))^T \, | \, (s, a) \sim d_\gamma^{\pi_\theta}, s' \sim \mathcal{P}(\cdot | s, a), a' \sim \pi_\theta(\cdot | s')]
\end{align*}
% \todo[inline]{AAJ: rewrite A, g}
% $A = \mathbb{E}_{{(s, a) \sim d_\gamma^{\pi_\theta}, s' \sim \mathcal{P}(\cdot | s, a), a' \sim \pi_\theta(\cdot | s')}} [A_{k+1}]$ and $g = \mathbb{E}_{\substack{(s, a) \sim d_\gamma^{\pi_\theta}\\ s' \sim \mathcal{P}(\cdot | s, a), a' \sim \pi_\theta(\cdot | s')}} [g_{k+1}] $
To prove the convergence of this linear TD update rule, we make a few standard assumptions. 
% \begin{assumption}\label{ass:1}
%    The matrix $\Psi$ has linearly independent rows. 
% \end{assumption}
% \begin{assumption}\label{ass:2}
%     The matrix $A$ is non-singular 
% \end{assumption}
% \begin{assumption}\label{ass:3}
%     The feature matrix $\Phi$ has uniformly bounded second moments. 
% \end{assumption}
\begin{assumption}\label{ass:1}
   \begin{enumerate}
       \item The matrix $\Psi$ has linearly independent rows. 
       \item The matrix $A$ is non-singular.
       \item The feature matrix $\Phi$ has uniformly bounded second moments.
   \end{enumerate}
\end{assumption}
\begin{theorem}\label{le:convergence_1}
    Under Assumption \ref{ass:1} and the fact that the learning rate satisfies the Robbins Monroe condition \cite{10.1214/aoms/1177729586} then the
    TD update equation \ref{eq:linear_td} converges in probability to the solution
    \begin{align*}
        \lim_{k \rightarrow \infty} \zeta_k = -A^{-1} b 
    \end{align*}
\end{theorem}
\todo[inline]{1. "the fact that learning rate..."? has it been assumed before? maybe we could add it as another assumption after the three previous ones. 2. what is $b$\\
\textbf{Answer} Makes sense.. Can do this}
\begin{proof}
The proof follows similar to \cite[Theorem 2]{DBLP:journals/corr/abs-2001-11113} which invokes ~\cite[Theorem 2.2]{10.1137/S0363012997331639}. We first re-write the updated equation in the following form
\begin{align*}
    \zeta_{k+1} = \zeta_{k} + \alpha_{k}(A \zeta_k + g + (A_{k+1}-A)\zeta_k + (g_{k+1} - g)).
\end{align*}
The proof follows almost equivalently from \cite[Theorem 2]{DBLP:journals/corr/abs-2001-11113} since the assumptions are the same. 
The above equation has separated the update into the deterministic part, $h(\zeta_k) :=A \zeta_k + g$ and the Martingale part, $M_{k+1} := (A_{k}-A)\zeta_k + (g_k - g)$. To apply this theorem, we now need to show that the function $h(\zeta)$ is asymptomatically stable. For stability, we will now show that the matrix $A$ has all negative Eigenvalues. Consider any unit eigenvector $x$. We will now try to evaluate $x^T A x$. 
\begin{align*}
    x^T A x&= x^T \, \EE_{(s, a) \sim d_\gamma^{\pi_\theta}}[-\Phi(s, a) \Phi(s, a)^T + \gamma \Phi(s, a) \EE_{s' \sim \P(\cdot | s, a), a' \sim \pi_\theta(\cdot | s')}[\Phi(s', a')]] \, x  \\
    &= -\EE_{(s, a) \sim d_\gamma^{\pi_\theta}}[x^T\Phi(s, a) \Phi(s, a)^T x] + \gamma [x^T \Phi(s, a) \EE_{s' \sim \P(\cdot | s, a), a' \sim \pi_\theta(\cdot | s')}[\Phi(s', a')]]
\end{align*}
Using Cauchy-Shwartz inequality on the second term. We now get the following, 
\begin{align*}
    x^T A x  \leq -(1-\gamma)\EE_{(s, a) \sim d_\gamma^{\pi_\theta}}[x^T\Phi(s, a) \Phi(s, a)^T x] \leq 0
\end{align*}
Since $A$ is non-singular (Assumption \ref{ass:1}), the eigenvalues cannot be zero. Thus, the eigenvalues for this problem are strictly negative. This guarantees asymptotic stability of $h(\zeta)$ and completes the proof. 
\end{proof}
\subsection{Proof of Theorem \ref{le:convergence_minmax}}\label{proof:convergence_minmax}
\begin{proof}
The proof follows similar to to \cite[Theorem 2]{DBLP:journals/corr/abs-2001-11113} which invokes ~\cite[Theorem 2.2]{10.1137/S0363012997331639}. 
% We first re-write the updated equation in the following form. 
We first re-write the updates in equation \ref{eq:update_equation} in matrix form as $d_{t+1} = d_{t} + \varepsilon_t (G_{t+1} d_t + h_{t+1})$ with $d_t := [\alpha_t, \beta_t, \tau_t^T ]$ and $G_{t+1}, h_{t+1}$ are as follows, 
\begin{align*}
    G_{t+1} &= \begin{bmatrix}
    0 & -A_t & -\lambda \Phi_t  \\ A_t & -C_t & 0 \\
    \lambda \Phi_t^T & 0 & -\lambda 
    \end{bmatrix}, \quad h_{t+1} = \begin{bmatrix}
    0 \\
    -B_t \\
    0
    \end{bmatrix}
\end{align*}
where, $A_t = (\Phi_t \Phi_t^T - \gamma \Phi_t (\Phi_t')^T), B_t = \Phi g_t^T, C_t = \Phi_t \Phi_t^T$. 
%Since this matrix is for each element in the dataset sampled from $d_\gamma^{\pi_\theta}$. 
We can calculate the expectation for each of these matrices as follows (the expectation is taken over all possible values of $(s,a)$ distributed as $d_{\gamma}^{\pi_{\theta}}$), 
\begin{align*}
    G = \mathbb{E}_{p}[G_{t+1}] = \begin{bmatrix}
        0 & -A & -\lambda \D^T \Phi \\
        A & C & 0 \\
        \lambda \Phi^T \D & 0 & -\lambda
    \end{bmatrix}, \quad  h = \mathbb{E}_{(s, a) \sim d_\gamma^{\pi_\theta}}[h_{t+1}] = \begin{bmatrix}
        0 \\
        -B \\
        0
    \end{bmatrix}
\end{align*}
Where, 
\begin{gather*}
    A := \Psi(I - \gamma P_{\pi_\theta})\D \Psi^T, \quad B := \Psi \D G^T, \quad C := \Psi \D \Psi, \\
    E_p[\cdot] := \E_{(s, a) \sim d_\gamma^{\pi_\theta}, 
    s' \sim \mathcal{P}(\cdot | s, a), a' \sim \pi_\theta(\cdot | s')}[\cdot]
\end{gather*}
where $\D$ is a diagonal matrix with the diagonal being $d_{\gamma}^{\pi_{\theta}}$. 
We will now prove the convergence of the linear function approximation case.
We first separate the deterministic term $h(d_t)$ and the stochastic term $M_{t+1}$
\begin{align*}
    d_{t+1} = d_{t} + \varepsilon_t(\underbrace{Gd_{t} + g}_{h(d_t)} + \underbrace{(G_{t+1} - G)d_{t} + (g_{t+1} - g))}_{M_{t+1}}
\end{align*}
From here, the proof follows equivalently to ~\cite[Theorem 2]{DBLP:journals/corr/abs-2001-11113}, with exactly same assumptions. The remaining part of the proof requires us to show that the function $h(d)$ is symptomatically stable. We show that by demonstrating that the eigenvalues of the matrix $G$ are strictly negative. 
% Let $v = [v_1, v_2, v_3]$ be an eigenvector of the matrix $G$. We can thus write that the Eigen value $\mu$ of the matrix $G$
Let $v$ be an arbitrary eigenvector of $G$ corresponding to an arbitrary eigenvalue $\nu$, then
\begin{align*}
    \mu &= v^T G v \\
    &= v_2^T A v_1 + \lambda v_3^T \Phi^T \D v_1 - v_1^T A v_2 - \lambda v_1^T \Phi^T \D v_3 -v_2^T C v_2 - \lambda v_3^T v_3 \\
    &= -\lambda v_3^T v_3  - v_2^T C v_2 \leq 0
\end{align*}
Now, we just need to show that the eigenvalue is not zero. We will prove that using contradiction. Assume that there exists a $v$ such that $Gv = 0$ and $v \neq 0$. This implies that, 
\begin{align*}
    &-Av_2 - \lambda \Phi^T \D v_3 = 0 \\
    &Av_1 = 0 \\
    &\lambda \D^T\Phi v_1 - \lambda v_3 = 0
\end{align*}
Since $A$ is a non-singular matrix (Assumption \ref{ass:3}) we have $v_1 = 0$ which implies that $v_3 = 0$ from the third equation. This leaves us with $A v_2 = 0$ which leaves us with $v_2 = 0$, which is a contradiction. This finishes the proof. It now is proven that the eigenvalue is strictly negative. This completes the proof. 
\end{proof}
%\subsection{Proof of Lemma \ref{le:linear_case}}\label{sec:linear_case_proof}
\subsection{Proof of Theorem \ref{th:sample_complexity}}\label{sec:additional_theory}
% \begin{theorem}
% Under assumptions \ref{ass:1}-\ref{ass:4}, we can upper bound the sample complexity using the optimality gap $\epsilon (\bar{y}. \bar{w})$ as follows, 
% \begin{align}
%     \epsilon(\bar{\alpha}, \bar{\beta}, \bar{\tau}) =  \max_{(\beta, \tau) \in {Y \times Z}} L(\bar{\alpha}, \beta, \tau) - \min_{\alpha' \in X} L(\alpha', \bar{\beta}, \bar{\tau}) \leq C_0\sqrt{\frac{5}{n}}(8 + 2 \log \frac{2}{\delta})
% \end{align}
% \end{theorem}
\begin{proof}
The proof of this Theorem is almost similar to the proof of ~\citet{DBLP:journals/corr/abs-2001-11113} except the need to bound different terms. To simplify the proof, we first lump the maximization and minimization variables together and update these parameters in matrix form. We call the grouped maximization variables $y = (\beta, \tau)$ and minimization variable $\alpha$. We re-write Algorithm \ref{alg:dice_ldg} in matrix form as follows (in Algorithm \ref{alg:dice_ldg_2}), 
\begin{algorithm}[!htb]
  \caption{Projected Log Density Gradient}\label{alg:dice_ldg_2}                          
  \begin{algorithmic}[1]
    \STATE \textbf{for} {$i = 1, 2, ..., n$ do:}
    \STATE $\alpha_{t+1} = \Pi_{X}(\alpha_t - \varepsilon_t(G_{1, t}y_t))$
    \STATE $y_{t+1} = \Pi_{Y, Z}(y_{t+1} + \varepsilon_t(G_{2, t}\alpha_t + G_{3. t}y_t + G_{4, t})$
    \STATE \textbf{Return} $\bar{\alpha}, \bar{y}$\\
    Where, $\bar{\alpha} = \frac{\sum_{i=1}^n \varepsilon_i \alpha_i}{\sum_{{i=1}}^{n}\varepsilon_i}$ , $\bar{y} = \frac{\sum_{i=1}^n \varepsilon_i y_i}{
    \sum_{i=1}^{n} \varepsilon_i}$
    
  \end{algorithmic}
\end{algorithm}
where, 
\begin{align*}
&G_{1, t} := \begin{bmatrix}
    -(\Phi_t\Phi_t^T - \gamma \Phi_t \Phi_t'^T) & -\lambda \Phi_t \\
\end{bmatrix} \\
&G_{2, t} := \begin{bmatrix}
    \Phi_t \Phi_t^T - \gamma \Phi_t \Phi_t^{'T} \\
    \lambda \Phi_t^T
\end{bmatrix}, G_{3, t} := \begin{bmatrix}
    -\Phi_t \Phi_t^T & 0 \\
    0 & -\lambda 
\end{bmatrix} \\
&G_{4, t} := \begin{bmatrix}
        -\Phi_{t} g_{t}^T \\ 
        0 
\end{bmatrix}.
\end{align*}
We also project our variables $\alpha_{t}, y_{t}$ on the closed and convex sets $X \subset \mathbb{R}^{d \times n}, Y \subset \mathbb{R}^{d \times n} \times \mathbb{R}^{1 \times n}$. We output the weighted average $\bar{\alpha}$ and $\bar{y}$. Before, proposing a sample complexity bound, we first define the optimization gap $\epsilon_g(\alpha, y)$ as follows, 
\begin{align*}
    \epsilon_g(\alpha, y) := \max_{y' \in {Y}} L(\alpha, y') - \min_{\alpha' \in X} L(\alpha', y)
\end{align*}
Note that $\epsilon_g(\alpha_*, y_*) = 0$, where $\alpha_*, y_*$ are the solutions to the min-max problem. From here, on the proof follow almost similarly to ~\citet{DBLP:journals/corr/abs-2001-11113}. The proof follows from \cite[Proposition 3]{DBLP:conf/uai/LiuLGMP15} and  \cite[Proposition 2]{DBLP:journals/corr/abs-2001-11113}, both of which rely on  \cite[Proposition 3.2]{DBLP:journals/siamjo/NemirovskiJLS09} to state the $O\left(\sqrt{\frac{1}{n}}\right)$ bound.
    % The proof is similar to Proposition 3 \cite{DBLP:conf/uai/LiuLGMP15} and Proposition 2 \cite{DBLP:journals/corr/abs-2001-11113}. Essentially both of these two proofs including ours will attempt to use Proposition 3.2 \cite{DBLP:journals/siamjo/NemirovskiJLS09} and fit the problem in that framework. 
    Proposition 3.2 in \cite{DBLP:journals/siamjo/NemirovskiJLS09} says that set $X, Y, Z$ should be closed, convex and bounded sets, which is part of our assumption. Our min-max loss function is Lipschitz continuous and the minimization problem is convex and the maximization problem is concave. It can also be seen that both the primal and dual form of the optimization has equal optimal values $\alpha_*, y_*$. We now proceed to apply the bound proposed by \cite{DBLP:journals/siamjo/NemirovskiJLS09}. To that end, we first need to bound certain terms. We define $D_{\alpha}$ and $D_{Y}$ as follows, 
\begin{align*}
    D_{\alpha} = \max_{x \in X} \|x\|^2 - \min_{x \in X} \|x\|^2\\
    D_{Y} = \max_{y \in Y} \|y\|^2 - \min_{y \in Y} \|Y\|^2
\end{align*}
From Assumption \ref{ass:3} which says the second moment of all the features is bounded  we can similarly write the following bound
\begin{align*}
    \EE[\|G_{i, t} - G_{i}\|^2 ] \leq \sigma_i^2 \quad \forall i \in \{1, 2, 3, 4\}.
\end{align*}
Therefore, we obtain bounds for the stochastic sub-gradient $G_{\alpha}(\alpha, y), G_{Y}(\alpha, y)$ as follows, 
\begin{align*}
    &G_{\alpha}(\alpha, y) = G_{1, t}y_t \\
    &G_{Y}(\alpha, y) = (G_{2. t}\alpha_t + G_{3, t}y_t + G_{4, t})
\end{align*}
and for their second moment as follows, 
\begin{align*}
    \EE[\| G_{\alpha}(\alpha, y)\|^2] &= \sigma_1^2 D_Y^2 + \sigma_1^2 \|\bar{G}_{1, t}\|^2 \leq C_\alpha \\
    \EE[\| G_{Y}(\alpha, y)\|^2] &= \sigma_2^2 D_Y^2 + \sigma_2^2 \|\bar{G}_{2, t}\|^2 + \sigma_3^2 D_Y^2 + \sigma_3^2 \|\bar{G}_{3, t}\|^2 + \sigma_4^2\|\bar{G}_{4, t}\|^2\leq C_Y, \\
\end{align*}
where we used the fact that $\EE[\| x \|^2] \leq \EE[\|x - \EE[x]\|^2] + \|\EE[x]\|^2$. If we now follow the procedure as proposed by \cite{DBLP:journals/siamjo/NemirovskiJLS09} we can define $M_*$ as follows, 
\begin{align*}
    M_*^2 = 2 C_\alpha^2 D_{\alpha^2} + 2 C_Y^2 D_Y^2
\end{align*}
If we fix the learning rate as $\varepsilon_t = \frac{c}{M_* \sqrt{t}}$ for any positive constant $c$. We can now bound the optimality gap with probability at least $1-\delta$ using  \cite[Proposition 3.2]{DBLP:journals/siamjo/NemirovskiJLS09}
\begin{align*}
     \epsilon(\bar{\alpha}, \bar{y}) \leq \frac{5}{n}(8 + 2 \log \frac{2}{\delta})M_* \max\{c, \frac{1}{c}\}
\end{align*}
This completes the proof. 
\end{proof}
\end{document}